\documentclass{article}

\usepackage{geometry}          
\usepackage[sort&compress,numbers]{natbib}
\usepackage[utf8]{inputenc} % allow utf-8 input
\usepackage[T1]{fontenc}    % use 8-bit T1 fonts
\usepackage[colorlinks=true, linkcolor=red, citecolor=teal, urlcolor=blue, breaklinks=true, bookmarksopen=true]{hyperref}
\usepackage{url}            % simple URL typesetting
\usepackage{booktabs}       % professional-quality tables
\usepackage{amsfonts}       % blackboard math symbols
\usepackage{nicefrac}       % compact symbols for 1/2, etc.
\usepackage{microtype}      % microtypography
\usepackage{xcolor}         %

\usepackage{subcaption}      % for subfigure support

\newcommand{\proofof}[1]{\textbf{\textup{Proof of \Cref{#1}}}}

\usepackage{amsmath, amsthm, mathtools, bm, nicefrac, amssymb}

\newtheorem{theorem}{Theorem}
\newtheorem{proposition}{Proposition}
\newtheorem{lemma}{Lemma}
\newtheorem{corollary}{Corollary}
\newtheorem{observation}{Observation}
\newtheorem{remark}{Remark}

\theoremstyle{definition}
\newtheorem{definition}{Definition}
\newtheorem{example}{Example}

\usepackage[capitalize,noabbrev]{cleveref}
\Crefname{observation}{Observation}{Observations}
\Crefname{example}{Example}{Examples}
\Crefname{lemma}{Lemma}{Lemmas}
\Crefname{theorem}{Theorem}{Theorems}
\Crefname{proposition}{Proposition}{Propositions}
\Crefname{corollary}{Corollary}{Corollaries}

\usepackage{algorithm}
\usepackage[noend]{algorithmic}

\usepackage{kbordermatrix} % For labeled matrices
\DeclarePairedDelimiter\abs{\lvert}{\rvert}
\DeclarePairedDelimiter\norm{\lVert}{\rVert}
\DeclareMathOperator*{\argmax}{arg\,max}

\newcommand{\E}[1]{\mathbb{E}\left[ #1 \right]}
\newcommand{\prob}{\textsc{Rec-APC}}
\newcommand{\ovv}{\overline{V}}
\newcommand{\unv}{\underline{V}}
\newcommand{\bb}{\bm b}
\newcommand{\bP}{\bm P}
\newcommand{\bq}{\bm q}
\newcommand{\vs}{V^\star}
\newcommand{\ps}{\pi^\star}

\usepackage{graphicx}
\usepackage{subcaption}   % For subfigures

\newcommand{\IFTHEN}[2]{\STATE \textbf{if} #1 \textbf{then} #2}

\makeatletter
\newcommand{\alglinelabel}[1]{%
  \addtocounter{ALC@line}{-1}
  \refstepcounter{ALC@line}
  \label{#1}
}
\makeatother

\makeatletter
\renewcommand{\paragraph}{%
  \@startsection{paragraph}{4}{\z@}%
  {0em}% No space before
  {-1em}% Negative space after (pulls following text up)
  {\normalfont\normalsize\bfseries}% Formatting
}
\makeatother

\title{Churn-Aware Recommendation Planning under Aggregated Preference Feedback}

\author{
Gur Keinan%
\thanks{%
    {Technion---Israel Institute of Technology (\url{gur.keinan@campus.technion.ac.il})}}
\and Omer Ben{-}Porat%
\thanks{%
    {Technion---Israel Institute of Technology (\url{omerbp@technion.ac.il})}}
}

\date{}

\begin{document}

\maketitle

\begin{abstract}
We study a sequential decision-making problem motivated by recent regulatory and technological shifts that limit access to individual user data in recommender systems (RSs), leaving only population-level preference information. This privacy-aware setting poses fundamental challenges in planning under uncertainty: Effective personalization requires exploration to infer user preferences, yet unsatisfactory recommendations risk immediate user churn. To address this, we introduce the $\prob$ model, in which an anonymous user is drawn from a known prior over latent user types (e.g., personas or clusters), and the decision-maker sequentially selects items to recommend. Feedback is binary---positive responses refine the posterior via Bayesian updates, while negative responses result in the termination of the session.

We prove that optimal policies converge to pure exploitation in finite time and propose a branch-and-bound algorithm to efficiently compute them. Experiments on synthetic and MovieLens data confirm rapid convergence and demonstrate that our method outperforms the POMDP solver SARSOP, particularly when the number of user types is large or comparable to the number of content categories. Our results highlight the applicability of this approach and inspire new ways to improve decision-making under the constraints imposed by aggregated preference data.
\end{abstract}

\section{Introduction}

Recommender systems (RSs) have become essential in digital media and e-commerce. They collect massive amounts of data and apply sophisticated techniques to improve user engagement and satisfaction. RSs leverage past user interaction, demographics, and possibly additional information to provide personalized recommendations. For instance, Netflix and Amazon utilize these techniques to offer tailored movie recommendations and product suggestions, respectively. While RSs typically collect and record user data, some have limited access to individual user information or none at all. For instance, regulations such as the General Data Protection Regulation (GDPR) and the California Consumer Privacy Act (CCPA) impose stringent rules on data usage and user consent. Privacy prioritization has also become a trend among commercial companies; for example, Apple's iOS 14 opt-in device tracking modification significantly impacted targeted advertising~\citep{kollnig2022goodbye}. In some cases, the RS can record \emph{user sessions} of varying length, but cannot identify users. While these regulatory shifts aim to protect privacy, they pose substantial challenges to RSs in maintaining the same level of service. Due to these privacy-preserving limitations, RSs may at times be limited to utilizing \emph{aggregated user information}, such as clusters of users or personas.

Relying solely on aggregated information significantly affects the recommendation process. When interacting in a user session without accurate individual data, the RS must explore more intensively to understand the current user's preferences. This initial exploration phase is critical for gathering enough data to make accurate recommendations later in the session. However, aggressive exploration bears the risk of \emph{user churn}, where users may leave the system due to receiving unsatisfactory recommendations. For example, when a song-recommendation RS encounters a user with unknown preferences, it may suggest a song from an unconventional genre, which some users enjoy, while most dislike. Although user feedback on that unconventional genre can provide valuable insights into their preferences, it can cause users to leave the RS if the recommendation is off-target; thus, the RS should address the possibility of churn as part of its design.

In this paper, we propose a model to study the intertwined challenges of aggregated user information and churn risk. We call our model $\prob$, standing for \textbf{Rec}ommendation with \textbf{A}ggregated \textbf{P}references under \textbf{C}hurn. Our model assumes that the RS has a probabilistic prior over user types (clusters, personas, etc.) and aggregated satisfaction levels for each content type (genre, etc.) with respect to each user type. Each user session involves an unidentified user whose type is sampled from this prior distribution and is unknown to the RS. The RS recommends content sequentially, receiving binary feedback (i.e., "like" or "dislike") from the user. This feedback enables the RS to infer the user type in a Bayesian sense, thereby improving future recommendations. However, careless recommendations can lead to user churn. The objective of the RS is to maximize user utility, defined as the number of contents the user likes.

\subsection{Our contribution}
Our contribution is three-fold. First, this work is among the first to address aggregated user preferences and the risk of churn simultaneously (see related work below). We propose a model where the general population preferences are known and the type (cluster, persona, etc.) of users is hidden but can be deduced through interaction. This interaction drives user engagement, which we model as the RS's reward. However, due to uncertainty about user type, the RS may generate unsatisfactory recommendations that can lead to user churn. Our model presents a novel exploration-exploitation trade-off, particularly regarding the risk of churn, an aspect that has been under-explored in the current literature.

Our second contribution is technical. Analyzing the model shows that optimal policies converge after a finite number of recommendations, symbolizing a transition to pure exploitation. Informally, we focus on a broad class of \emph{well-separated} instances (see Definition~\ref{def:well separated}), and show that
\begin{theorem}[Informal statement of \Cref{thm:convergence}]
The (infinite-length) optimal policy converges.
\end{theorem}

Our third contribution is algorithmic. We leverage our theoretical results to develop a straightforward yet state-of-the-art branch-and-bound algorithm designed explicitly for our setting. As the problem we address can be described as a partially observable Markov decision process (POMDP), we compare our algorithm with the state-of-the-art POMDP benchmark~\citep{kurniawati2009sarsop}.\footnote{All code and experimental data for reproducing the results presented in this paper are available at \url{https://github.com/GurKeinan/Churn-Aware-Recommendation}} In practice, our algorithm performs better when there is a large variety of user types but is less effective when the number of contents is significantly larger than the number of user types.

\subsection{Related work}
Our work captures several phenomena: First, we assume that the RS has aggregated user information, but no access to individual user information. Second, we have sequential interaction, where we can balance exploration-exploitation trade-offs. And third, our model includes the risk of user churn. Below, we review the relevant literature strands.

\paragraph{Aggregated user information}
Our model follows the trend of using clustered data in the recommendation process~\citep{recommender-systems-for-large-scale-e-commerce-scalable-neighborhood-formation-using-clustering}. In addition to improved efficiency, the use of clustering can increase the diversity and reliability of recommendations \citep{a-clustering-approach-for-personalizing-diversity-in-collaborative-recommender-systems, robust-collaborative-filtering-based-on-multiple-clustering} and handle the sparsity of user preference matrices \citep{recommender-systems-clustering-using-bayesian-non-negative-matrix-factorization}.

Aggregated information also relates to privacy, a topic that has gained much attention recently, following the seminal work of \citet{differential-privacy} on differential privacy. Several works propose RSs that satisfy differential privacy \citep{differentially-private-recommender-systems-building-privacy-into-the-netflix-prize-contenders, differentially-private-collaborative-coupling-learning-for-recommender-systems, differential-privacy-for-collaborative-filtering-recommender-algorithm}. In a broader context, \citet{the-effect-of-online-privacy-information-on-purchasing-behavior-an-experimental-study} have empirically shown that users value their privacy and are willing to pay for it. Several other definitions of privacy were suggested in the literature \citep{an-agent-based-approach-for-privacy-preserving-recommender-systems,
enhancing-privacy-and-preserving-accuracy-of-a-distributed-collaborative-filtering,svd-based-collaborative-filtering-with-privacy}. In our work, we assume that the RS has access to aggregated data, akin to other recent works addressing lookalike clustering~\citep{anonymous-learning-via-look-alike-clustering-a-precise-analysis-of-model-generalization, interactive-and-explainable-point-of-interest-recommendation-using-look-alike-groups}.

The cases where the RS has little information about users or items are referred to as \emph{cold-start} problems. This issue relates to our work because, while we assume access to aggregated information, every user interaction starts from a tabula rasa. Broadly, solutions are divided into data-driven approaches~\citep{a-heterogeneous-information-network-based-cross-domain-insurance-recommendation-system-for-cold-start-users, transfer-meta-framework-for-cross-domain-recommendation-to-cold-start-users, alleviating-data-sparsity-and-cold-start-in-recommender-systems-using-social-behaviour} and method-driven approaches~\citep{personalized-adaptive-meta-learning-for-cold-start-user-preference-prediction, task-adaptive-neural-process-for-user-cold-start-recommendation, meta-matrix-factorization-for-federated-rating-predictions} (see \citet{user-cold-start-problem-in-recommendation-systems-a-systematic-review} for a recent survey). 

This paper is \emph{inspired} by clustering, privacy, and cold-start problems. However, our model only assumes access to aggregated information and abstracts the reasons why individual information is unavailable. Notably, we do not propose techniques to cluster users, address privacy concerns, or provide new approaches for the cold-start problem.

\paragraph{Sequential recommendation with churn}
Our model falls under Markov Decision Process (MDP) modeling for RSs~\citep{an-mdp-based-recommender-system}, where the belief over user types represents the state. Alternatively, we can formalize it as a Partially Observable MDP (POMDP), where the state corresponds to the user type that remains constant but is initially unknown. Both MDP and POMDP modeling are well-studied in the literature of RSs~\citep{optimal-recommendation-to-users-that-react-online-learning-for-a-class-of-pomdps, interactive-recommendation-with-user-specific-deep-reinforcement-learning, recommendation-as-a-stochastic-sequential-decision-problem}, and typically the main task is to learn the underlying model~\citep{empirical-evaluation-of-gated-recurrent-neural-networks-on-sequence-modeling, usage-based-web-recommendations-a-reinforcement-learning-approach}.

We model user churn as part of the sequential recommendation process, thereby generating an exploration-exploitation trade-off. User churn is an integral part of many systems, and most of the literature addresses user churn prediction~\citep{user-retention-a-causal-approach-with-triple-task-modeling,quantifying-and-leveraging-user-fatigue-for-interventions-in-recommender-systems} and techniques to retain users~\citep{e-government-deep-recommendation-system-based-on-user-churn,surrogate-for-long-term-user-experience-in-recommender-systems}. Several recent works~\cite {maximizing-cumulative-user-engagement-in-sequential-recommendation-an-online-optimization-perspective, returning-is-believing-optimizing-long-term-user-engagement-in-recommender-systems, partially-observable-markov-decision-process-for-recommender-systems, cao2020fatigue} adopt a similar approach to ours, directly modeling user churn due to irrelevant recommendations. 

The paper most relevant to ours is the one by \citet{modeling-attrition-in-recommender-systems-with-departing-bandits}. They focus on the problem of online learning with the risk of user churn under user uncertainty. While they study the stochastic variant, where the user-type preferences are initially unknown, their analysis is restricted to one user and multiple categories, or two users and two categories. In contrast, we assume complete information, but we study general matrices of any size.

\section{Problem Definition}
\label{sec:problem-definition}

In this section, we formally introduce our model alongside the optimization problem the recommender system aims to solve. We then provide an illustrative example and explain why this problem necessitates careful planning by showing the sub-optimality of naive and seemingly optimal solutions.

\subsection{Our Model}
In this section, we formally define the \textbf{Rec}ommendation with \textbf{A}ggregated \textbf{P}references under \textbf{C}hurn, or $\prob$ for abbreviation. An instance of the problem is fully specified by the tuple $\mathcal{I} = \langle M, K, \bq, \bP \rangle$ whose components we describe below.

We consider a set $U$ of users who interact with the RS. The RS does not have access to individual user information; instead, it relies on aggregated data about users. Specifically, we assume that there is a finite set $M$ of \emph{user types}, with each type representing a persona, i.e., a cluster of homogeneous users with similar preferences. Each user in $U$ is associated with precisely one user type in $M$ according to its preferences. Accordingly, we denote by $m(u)\in M$ the type of user $u\in U$. The RS has a prior distribution $\bq$ on the elements of $M$, that is $\bq \in \Delta(M)$. This prior reflects the likelihood that a new user belongs to each type, incorporating factors such as the proportion of users in each type and their respective arrival frequencies. We assume that all entries in $\bq$ are strictly positive. 

Analogously, our model abstracts contents into broader \emph{categories}, each representing a group of similar items. This abstraction enables the RS to make recommendations from a finite set $K$ of categories.

Lastly, there is a user-type preferences matrix $\bP \in [0,1]^{K \times M}$. Each element $\bP(i,j)$ signifies the likelihood that category $i$ satisfies a user of type $j$. We stress that $\bP$ contains information on user types but not on specific users. The RS has complete information on the user-type preference matrix $\bP$ and prior user-type distribution $\bq$.

\begin{remark}
    While we introduce a set $U$ of users to motivate the semantics of the model, the formal instance definition $\mathcal{I} = \langle M, K, \bq, \bP \rangle$ does not include $U$, as individual users are never observed and play no role in the system dynamics. All aspects of the interaction are fully determined by the tuple $\mathcal{I} = \langle M, K, \bq, \bP \rangle$.
\end{remark}

\paragraph{User session}
A user session starts when a user $u \in U$ enters the RS. The RS lacks access to any information about $u$; thus, it only knows that $m(u)$ is distributed according to $\bq$. The session consists of rounds. In each round $t$, for $t\in \{1,\dots,\infty\}$, the RS recommends a category $k_t \in K$. Afterward, the user provides binary feedback: They either like the item, with a probability of $P(k_t, m(u))$, or they dislike it with the complementary probability. If the user likes the recommended category, the RS receives a reward of $1$, and the session continues for another round. However, when the user ultimately dislikes a recommended category, the RS earns a reward of $0$, and the session concludes as the user leaves the RS.

\paragraph{Recommendation policy}
The RS produces recommendations according to a recommendation \emph{policy}. Recommendation policies can depend solely on the current user session and the history within the session. That is, in round $t$ of the session, the policy can depend on histories of the form $\left(K,\{0,1\} \right)^{t-1}$, where each tuple comprises a recommended category and its corresponding binary feedback. However, since feedback of zero (dislike) leads to ending the session, we can succinctly represent a policy as a sequence of recommendations. Namely, we represent a policy $\pi$  as the infinite series of categories $\pi = (k_t)_{t=1}^{\infty}$. 
Observe that $k_t$ will only be recommended if the user liked $k_{t-1}$, as otherwise the user would leave. For convenience, we use $\pi[t:]$ to denote the recommendation sequence of $\pi$ starting from the $t$'th round, that is, $\pi[t:] = \left( k_{i} \right)_{i=t}^{\infty}$.

\paragraph{Social welfare}
We describe the utility of each user type as the count of times it engages with the platform and provides positive feedback, i.e., the number of likes. Let $F(m(u),\pi)$ represent the r.v. that counts the number of likes given by a user of type $m(u)$. This count is influenced by both the user type $m(u)$ and the recommendation policy $\pi$. We define $V^\pi$ as the expected social welfare, which is the mean utility of users under the recommendation policy $\pi$. That is,
\[
    {V^\pi = \mathbb{E}{\frac{1}{\abs{U}}\sum_{u \in U} F(m(u),\pi)} = \sum_{m \in M}\bq(m)\E{F(m,\pi)}},
\]
Where the equality follows from the definition of the prior $\bq$. To be consistent with the literature, we refer to $V^\pi$ as the \emph{value function}; we use both terms interchangeably.
As the goal of the recommender is to maximize expected social welfare, we define \emph{optimal policy} $\ps$ as $\ps \in \argmax_{\pi} V^\pi$ and the corresponding optimal expected social welfare $\vs$ as $\vs = V^{\ps}$.

\paragraph{Useful notation}
We emphasize that a policy $\pi$ is a function of $\bm {P}$ and $\bm {q}$, expressed as $\pi=\pi(\bP, \bq)$. For abbreviation, it is also convenient to denote this as $\pi(\bq)$ when $\bP$ is known from the context. This notation extends to any belief $\bb \in \Delta(M)$, not just the prior $\bq$. We make the same abuse of notation for the value function $V^\pi$. That is, we let $V^\pi(\bb)=\sum_{m \in M}\bb(m) \E{F(m,\pi)}$ and stress that $V^\pi(\bq)=V^\pi$. Additionally, we use the star notation $\ps(\bb), \vs(\bb)$ to denote optimal policy and value function w.r.t. the belief $\bb$. Finally, we let $p_{k}(\bb)$ denote the expected \emph{immediate reward}, specifically, the probability that a user likes category $k$ assuming that their type is drawn from $\bb$; that is, $p_{k}(\bb) = \sum_{m \in M} \bb(m) \bP(k, m)$.

\subsection{Bayesian Updates}
At the beginning of each user session, the RS is only informed about the prior $\bq$. However, it rapidly acquires more information about the user through their feedback. For instance, if a user likes a recommendation of an exotic category favored by only a small subset of user types, we can conclude that the user probably belongs to that subset of types. The RS can then update its \emph{belief} over the current user type, where the belief is a point in the user type simplex, and use it in its future suggestions. 

We employ Bayesian updates to incorporate the new information after user feedback. Starting from a belief $\bb$ and recommending a category $k$, we let $\tau(\bb, k)$ denote the new belief over user types in case of positive feedback. Namely, $\tau$ is a function $\tau: \Delta(M) \times K \to \Delta(M)$ such that for every belief $\bb \in \Delta(M)$, category $k \in K$ and type $m \in M$,
\begin{equation}\label{eq:bayesian update}
    \tau(\bb, k)(m) =  \frac{\bb(m) \cdot \bP(k,m)}{\sum_{m' \in M} \bb(m') \cdot \bP(k,m')}.
\end{equation}

We stress that Bayesian updates are crucial only after positive feedback, as the RS can utilize this new information for future recommendations. Negative feedback, though still informative, ends the session, preventing the RS from using the new information. We can further use Bayesian updates to obtain a recursive definition of the value function.
\begin{observation}
    \label{obs:recursive-formula-of-the-value-function}
    For every policy $\pi$ and belief $\bb \in \Delta(M)$,
    \[
        V^{\pi}(\bb) = p_{\pi_1}(\bb) \left( 1 + V^{\pi[2:]}(\tau(\bb, \pi_1)) \right).
    \]
\end{observation}
\Cref{obs:recursive-formula-of-the-value-function} provides a Bellman equation-like representation of the value function by isolating the immediate reward of the first round $p_{\pi_1}(\bb)$ and the future rewards (implicitly discounted by $p_{\pi_1}(\bb)$). It showcases the fundamental exploration-exploitation in our setting: On the one hand, exploitation involves selecting the category $k$ that currently maximizes $p_{k}(\bb)$, focusing on immediate reward. On the other hand, exploration aims to steer the updated belief $\tau(\bb, k)$ to a more informative position, thereby increasing future rewards. Next, we use the notion of belief walk to assess how policies navigate the belief simplex.

\begin{definition}[Belief Walk]
    The \emph{belief walk} induced by a policy $\pi$ starting at belief $\bb$ is the sequence $( \bb^{\pi, \bb}_t)_{t=1}^{\infty}$, where $\bb^{\pi, \bb}_1 = \bb$ and for all $t > 1$ , $\bb^{\pi, \bb}_t = \tau(\bb^{\pi, \bb}_{t-1}, \pi_{t-1}).$
\end{definition}

To enhance clarity, we provide geometric illustrations of belief walks induced by various recommendation policies in \Cref{sec:belief-walks}. Using this notion, we can delineate the value function with closed-form expressions, which will be useful in later analyses.

\begin{lemma}\label{lemma:closed-form-representations-of-the-value-function}
    For every policy $\pi$ and belief $\bb \in \Delta(M)$,
    \[
        V^\pi(\bb) = \sum_{t=1}^{\infty} \prod_{j=1}^{t} p_{\pi_j}(\bb^{\pi, \bb}_j) = \sum_{m \in M} \bb(m) \cdot \sum_{t=1}^{\infty} \prod_{j=1}^{t} \bP(\pi_j, m).
    \]
\end{lemma}

\begin{proof}[\proofof{lemma:closed-form-representations-of-the-value-function}]
    Let $T$ denote the number of likes received under policy $\pi$ before the user exits.  Note that each $p_{\pi_j}\bigl(\bb_j^{\pi,\bb}\bigr) = \sum_{m} \bb_j^{\pi,\bb}(m) \bP\bigl(\pi_j,m\bigr)$ is bounded by $p_{\max} = \max_{k,m} \bP(k,m) < 1$.\footnote{If $p_{\max} = 1$, that is, if there exist a category $k$ and user type $m$ such that $\bP(k,m) = 1$, then adopting the policy $\pi = (k)_{t=1}^\infty$ results in infinite expected social welfare. This is because with positive probability, a user of type $m$ will arrive and never leave the system, rendering the recommendation task trivial.} 
    This implies that the probability of staying in the system decreases geometrically with the number of rounds, 
    \[
        \mathbb{P}\bigl(T \ge t\bigr) =\prod_{j=1}^t p_{\pi_j}\bigl(\bb_j^{\pi,\bb}\bigr) \leq (p_{\max})^t,
    \]
    And that the sum of those probabilities is necessarily finite, as $\sum_{t=1}^\infty (p_{\max})^t = \frac{p_{\max}}{1-p_{\max}} < \infty$. In particular, $\mathbb{E}[T]$ must be finite, so the tail‐sum identity $
    \mathbb{E}[T] = \sum_{t=1}^\infty \mathbb{P}(T \ge t)$ is valid.  Therefore,
    \[
        V^\pi(\bb) = \mathbb{E}[T] = \sum_{t=1}^\infty \mathbb{P}(T \ge t) = \sum_{t=1}^\infty \prod_{j=1}^t p_{\pi_j}\bigl(\bb_j^{\pi,\bb}\bigr).
    \]
    Now, let $m \in M$ be a fixed user type. A user of type $m$ likes each recommendation $\pi_j$ with probability $\bP(\pi_j, m)$, and so the probability that such a user stays until round $t$ is $\prod_{j=1}^t \bP(\pi_j, m)$. From the same reasoning as above, the tail formula is valid to use on $F(m, \pi)$, and it yields:
    \[
        \mathbb{E}[F(m, \pi)] = \sum_{t=1}^{\infty} \mathbb{P}(F(m, \pi) \ge t) = \sum_{t=1}^{\infty} \prod_{j=1}^{t} \bP(\pi_j, m).
    \]
    Finally, taking expectation over the user type $m \sim \bb$ gives
    \[
        V^\pi(\bb) = \sum_{m \in M} \bb(m) \cdot \mathbb{E}[F(m, \pi)] = \sum_{m \in M} \bb(m) \cdot \sum_{t=1}^{\infty} \prod_{j=1}^{t} \bP(\pi_j, m),
    \]
    As claimed.
\end{proof}

\subsection{Illustrating Example and Sub-optimality of Myopic Recommendations}
\begin{example}\label{example:body}
    Consider a $\prob$ instance with two user types $M=\{m_1, m_2\}$ and two categories $K=\{k_1, k_2\}$. The preference matrix and the user type prior are:
    \renewcommand{\kbldelim}{(}% Left delimiter
    \renewcommand{\kbrdelim}{)}% Right delimiter
    \[
        \bP = \kbordermatrix{
            & m_1 & m_2  \\
            k_1 & 0.95 & 0.1  \\
            k_2 & 0.79 & 0.81
        }, \quad
        \bq =
        \kbordermatrix{
            &  \\
            m_1 & 0.5 \\
            m_2 & 0.5
        }.
    \]

    To interpret these values, note that user type $m_2$ likes category $k_1$ with a probability of $0.1$. In addition, recommending category $k_2$ to a random user yields an expected immediate reward of $p_{k_2}(\bq)=0.5 \cdot 0.79 + 0.5 \cdot 0.81 = 0.8$.

    A prudent policy to adopt is the \emph{myopic} policy $\pi^m$, which is defined to be the policy that recommends the highest yielding category in each round, and afterward updates the belief. In this instance, the myopic policy is $\pi^m = \left( k_2 \right)_{t=1}^{\infty}$, as $k_2$ provides the better expected immediate reward compared to $k_1$ for the prior, and the belief update will only increase the preference for $k_2$.
    Under this policy, the expected social welfare is the sum of two infinite geometric series, one for each user type; namely, $V^{\pi^m} = 0.5\cdot \frac{0.79}{1 - 0.79} +0.5 \cdot \frac{0.81}{1 - 0.81} =4.01$.

    While unintuitive at first glance, the optimal policy $\ps = \left( k_1 \right)_{i=1}^{\infty}$ achieves an expected value of $\vs = 0.5\cdot\frac{0.95}{1-0.95} + 0.5\cdot\frac{0.1}{1-0.1} = 9.55$, outperforming the myopic policy.
\end{example}
Beyond exemplifying our setting and notation, \Cref{example:body} demonstrates that myopic policies can be suboptimal. In fact, the next proposition demonstrates that the sub-optimality gap can be arbitrarily large, underscoring the necessity for thorough planning.

\begin{proposition}\label{prop:myopic-policy-suboptimality}
    For every $d \in \mathbb R_+$, there exists an instance \( \left\langle M, K, \bq, \bP \right\rangle \) such that
    \( V^{\pi^{m}} \cdot d \leq \vs \),
    where $\pi^{m}$ is the myopic policy for the instance.
\end{proposition}

\begin{proof}[\proofof{prop:myopic-policy-suboptimality}]
    Set some arbitrary $d > 0$, and consider the following $\prob$ instance:
    \[
        \left\langle
        K = \{ k_1, k_2 \}, M = \{ m_1, m_2 \},
        \bq = \begin{pmatrix}
            0.5 \\
            0.5
        \end{pmatrix},
        \bP = \begin{pmatrix}
            \frac{8d}{1+8d} & 0   \\
            0.8             & 0.8
        \end{pmatrix}
        \right\rangle.
    \]
    Notice that $p_{k_1}(\bq)= 0.5 \cdot \frac{8d}{1+8d} < 0.5 < 0.8 = p_{k_2}(\bq)$; thus, $\pi^m$ initially recommends $k_2$. Additionally, it holds that $\tau(\bq, k_2) = \bq$, as $\tau(\bq, k_2)(m_1) = \frac{0.5 \cdot 0.8}{0.5 \cdot 0.8+0.5 \cdot 0.8} = 0.5 = \bq(m_2)$; therefore, $\pi^m$ continues recommending $k_2$ indefinitely. Using \Cref{lemma:closed-form-representations-of-the-value-function}, we obtain $V^{\pi^m} = 0.5 \cdot \frac{0.8}{1-0.8} + 0.5 \cdot \frac{0.8}{1-0.8} = 4$. Now consider the constant policy $\pi = (k_1)_{t=1}^\infty$. Its expected value is $V^\pi = 0.5 \cdot \frac{\frac{8d}{1+8d}}{1 - \frac{8d}{1+8d}} = 4d$. Since $\ps$ is optimal, we conclude $\vs \geq V^\pi = 4d = d \cdot V^{\pi^m}$, thereby completing the proof.
\end{proof}

\section{Warm-Up: Algorithms for Rectangular instances}
\label{sec:approximating-the-optimal-policy}

In this section, we present two dynamic programming-based approaches that provide approximations of the optimal expected social welfare. Although inefficient in the general case, such solutions are useful when the number of types, $|M|$, or categories, $|K|$, is small. We refer to such instances as \emph{rectangular}. First, we define the notion of approximation through a finite horizon, and then present several methods for achieving it.

While our model lacks an explicit discount factor, \Cref{obs:recursive-formula-of-the-value-function} indicates that an implicit discount emerges through $p_{k}(\bb)$. Thus, similar to well-known results in MDPs, one can approximate the value function over an infinite horizon by addressing a finite-horizon problem with a sufficiently large horizon~\citep[Sec.~17.2]{ai-a-modern-approach}. To that end, we define the finite-horizon value function with a horizon $H$ as $V_H^{\pi}(\bb)=\sum_{m \in M}\bq(m)\E{\min \{H, F(m,\pi)\}}$. In other words, $V_H^{\pi}(\bb)$ is the value function given that the session terminates after $H$ rounds. \Cref{lemma:closed-form-representations-of-the-value-function} suggests that $V_H^{\pi}(\bb)$ is also given by $V_H^{\pi}(\bb) = \sum_{t=1}^{H} \prod_{j=1}^{t} p_{\pi_{j}}(\bb^{\pi, \bb}_j)$.

Next, we denote $p_{\max}$ as the largest entry in the matrix~$\bP$, assuming $p_{\max} < 1$, and use it to bound the gap between infinite and finite horizon optimal policies.

\begin{lemma}\label{lemma:finite-horizon-approximation}
    For any $\varepsilon >0$, it holds that $\vs \leq \max_{\pi'} V_{H(\varepsilon)}^{\pi'}+ \varepsilon$, where $H(\varepsilon) = \left\lceil \log_{p_{\max}} \frac{\varepsilon (1 - p_{\max})}{p_{\max}} \right\rceil$.
\end{lemma}

\begin{proof}[\proofof{lemma:finite-horizon-approximation}]
    Fix any $\varepsilon > 0$. Unrolling the recursive formula from \Cref{obs:recursive-formula-of-the-value-function} for $H$ steps, starting from the prior $\bq$, we obtain for any policy $\pi$:
    \[
    V^\pi = \sum_{t=1}^{H} \left( \prod_{j=1}^{t} p_{\pi_j}(\bb_j^{\pi, \bq}) \right) + \left( \prod_{j=1}^{H} p_{\pi_j}(\bb_j^{\pi, \bq}) \right) \cdot V^{\pi[H+1:]}(\bb_{H+1}^{\pi, \bq}).
    \]
    This expression decomposes the value function into two components: The total expected reward accumulated during the first $H$ rounds, which is precisely $V_H^\pi(\bq)$, and an additional term that captures the contribution from all subsequent rounds, multiplied by the probability that the user remains in the system after $H$ steps.
    
    Note that each immediate reward is at most $p_{\max}$, so $\prod_{j=1}^{H} p_{\pi_j}(\bb_j^{\pi, \bq}) \le p_{\max}^H$. Moreover, by \Cref{lemma:closed-form-representations-of-the-value-function}, the value function is uniformly bounded:
    \begin{equation}\label{eq:bounded-value}
        V^\pi(\bb) = \sum_{m \in M} \bb(m) \cdot \sum_{t=1}^{\infty} \prod_{j=1}^{t} \bP(\pi_j, m) \le \sum_{m \in M} \bb(m) \cdot \sum_{t=1}^{\infty} p_{\max}^t = \frac{p_{\max}}{1 - p_{\max}}.
    \end{equation}
    Thus, $V^\pi(\bq) \le V_H^\pi(\bq) + p_{\max}^H \cdot \frac{p_{\max}}{1 - p_{\max}}$. By choosing $H = \left\lceil \log_{p_{\max}} \left( \frac{\varepsilon (1 - p_{\max})}{p_{\max}} \right) \right\rceil$, the second term is at most $\varepsilon$. Therefore, for the optimal policy $\ps$,
    \[
    \vs(\bq) = V^{\ps}(\bq) \le V_H^{\ps}(\bq) + \varepsilon \le \max_{\pi} V_H^\pi(\bq) + \varepsilon.
    \]
    This completes the proof.
\end{proof}

\Cref{lemma:finite-horizon-approximation} reduces the task of finding an approximately optimal policy to finding an (exact or approximate) optimal policy for the finite-horizon case. For the rest of the section, we develop such solutions for a horizon of $H$.

\paragraph{Small number of categories}
Consider the brute-force approach, which evaluates all $K^H$ possible policies to find the optimal one. This becomes infeasible for large $H$, even with small $K$. However, we can exploit the order invariance of belief updates: The belief depends only on the multiset of selected categories, not their sequence. Thus, redundant calculations can be avoided by iteratively grouping sequences with the same multiset and retaining only the highest-performing sequence from each group.

\begin{proposition}\label{prop:backward-induction-approximation}
  We can find $\ps_H = \argmax_{\pi'} V^{\pi'}_H$ in a runtime of $O\left( (H + \abs{K}) ^ {\abs{K}+1} \cdot \abs{K} \cdot \abs{M} \right)$.
\end{proposition}

\begin{proof}[\proofof{prop:backward-induction-approximation}]
We first establish that belief updates commute under the Bayesian update rule. For any belief $\bb \in \Delta(M)$ and categories $k_1, k_2 \in K$, we have
\[
\tau(\tau(\bb, k_1), k_2)(m) = \frac{\bP(k_2, m) \cdot \tau(\bb, k_1)(m)}{\sum_{m'} \bP(k_2, m') \cdot \tau(\bb, k_1)(m')} = \frac{\bP(k_2, m) \cdot \bP(k_1, m) \cdot \bb(m)}{\sum_{m'} \bP(k_2, m') \cdot \bP(k_1, m') \cdot \bb(m')}.
\]
Since this expression is symmetric in $k_1$ and $k_2$, we conclude that $\tau(\tau(\bb, k_1), k_2) = \tau(\tau(\bb, k_2), k_1)$ for all $m \in M$. This commutativity implies that the belief state after $s$ positive feedback signals depends only on the multiset of recommended categories, not their sequential order. Consequently, let $\mathcal{M}_s$ denote the collection of all multisets over $K$ of cardinality $s$. For each $\bm{\mu} \in \mathcal{M}_s$, define $\bb_{\bm{\mu}}$ as the belief obtained from the prior $\bq$ by applying the Bayesian updates corresponding to the categories in $\bm{\mu}$. By commutativity, we have $\bb_{\bm{\mu} + \bm{e}_k} = \tau(\bb_{\bm{\mu}}, k)$ for any multiset $\bm{\mu} \in \mathcal{M}_s$ and category $k \in K$, where $\bm{e}_k$ denotes the unit multiset containing only category $k$, and $\bm{\mu} + \bm{e}_k$ represents adding one instance of category $k$ to multiset $\bm{\mu}$. Analogous to \Cref{obs:recursive-formula-of-the-value-function}, the optimal finite-horizon value function satisfies the Bellman equation $V_s(\bb_{\bm{\mu}}) = \max_{k \in K} \left\{ p_k(\bb_{\bm{\mu}}) \cdot \left(1 + V_{s-1}(\bb_{\bm{\mu} + \bm{e}_k}) \right) \right\}$.

The key insight is that due to commutativity, we need only compute $V_s(\bb_{\bm{\mu}})$ for each multiset $\bm{\mu} \in \mathcal{M}_s$ rather than considering all possible sequences of recommendations of length $s$. The computation proceeds by backward induction: For each $s$ from 1 to $H$, and for each multiset $\bm{\mu} \in \mathcal{M}_s$, we evaluate the recurrence over all $k \in K$ to determine the optimal category and corresponding value.

For the runtime analysis, observe that at each level $s$, we process $|\mathcal{M}_s| = \binom{s + |K| - 1}{|K| - 1} = O((s + |K|)^{|K|})$ distinct multisets. For each multiset $\bm{\mu}$ and each category $k \in K$, computing $p_k(\bb_{\bm{\mu}})$ requires $O(|M|)$ operations to evaluate the weighted sum over user types, and computing the belief $\bb_{\bm{\mu} + \bm{e}_k}$ similarly requires $O(|M|)$ operations. Therefore, the total computational cost is
\[
\sum_{s=1}^H O((s + |K|)^{|K|}) \cdot |K| \cdot |M| = O\left((H + |K|)^{|K|+1} \cdot |K| \cdot |M|\right).
\]
The optimal finite-horizon policy $\ps_H$ is then recovered by storing the maximizing categories at each step of the backward induction, and the proof is complete.
\end{proof}

\paragraph{Small number of user types}
If $M$ is small, we can employ a different approach. Recall that beliefs are points in the simplex $\Delta(M)$. Thus, we can discretize the belief simplex and execute dynamic programming. Specifically, we adopt the approach of \citet{point-based-value-iteration-an-anytime-algorithm-for-pomdps} and obtain:

\begin{proposition}\label{prop:discretization-approximation}
  We can find a policy $\pi$ that satisfies $V^{\pi}_H \geq \max_{\pi'} V^{\pi'}_H - \varepsilon$ in a runtime of \\ $O\left( H \cdot \abs{K} \cdot \abs{M} \cdot \left( \frac{1}{\varepsilon \cdot (1 - p_{\max})^2} \right)^{2\abs{M} - 2} \right)$.
\end{proposition}

\begin{proof}[\proofof{prop:discretization-approximation}]
We establish the result by constructing an auxiliary POMDP, demonstrating that it can be solved efficiently, and then relating its solution to $\prob$.

\paragraph{Auxiliary POMDP construction.} We construct a POMDP with state space $S = \{ m_s, m_f : m \in M \} \cup \{ \bot \}$, action space $K$, and discount factor $\gamma = p_{\max}$. The state space includes two copies of each user type: $m_s$ represents the initial interaction with a user of type $m$, while $m_f$ represents all subsequent interactions with that user. This separation is necessary because in the standard POMDP formulation, the first reward is undiscounted while subsequent rewards are discounted, which we need to align with our model's implicit discounting structure. The absorbing state $\bot$ represents user departure. The transition dynamics are: From state $m_s$, action $k$ transitions to $m_f$ with probability $\bP(k, m)$ and to $\bot$ with probability $1 - \bP(k, m)$; from state $m_f$, action $k$ transitions to $m_f$ with probability $\nicefrac{\bP(k, m)}{p_{\max}}$ and to $\bot$ with probability $1 - \nicefrac{\bP(k, m)}{p_{\max}}$; state $\bot$ is absorbing. Rewards are $0$ for transitions to $\bot$, otherwise $1$. Observations are deterministic: "like" for transitions to non-absorbing states, "dislike" for transitions to $\bot$. The initial belief assigns $\bb_0(m_s) = \bq(m)$ for all $m \in M$.

\paragraph{Value equivalence for corresponding policies.} Consider any finite policy $\pi = (k_1, \ldots, k_H)$ in $\prob$, and a corresponding POMDP policy $\tilde{\pi}$ that recommends action $k_{H-r+1}$ when $r$ rounds remain, independent of the current belief state. To calculate the expected value of this policy in the auxiliary POMDP, denoted as $V_{\text{POMDP}}^{\tilde{\pi}}$, we trace the expected value for a fixed user type $m$. Direct calculation yields $\sum_{t=1}^{\infty} \prod_{j=1}^t \bP(k_j, m)$, as the discounting factors cancel with transition probability adjustments starting from the second recommendation. This expression is identical to the corresponding value in \Cref{lemma:closed-form-representations-of-the-value-function}; therefore, $V^{\pi}(q) = V_{\text{POMDP}}^{\tilde{\pi}}(b_0)$ for any policy $\pi$ and its corresponding POMDP policy $\tilde{\pi}$.

\paragraph{Solving the POMDP and extracting a policy.} We use point-based value iteration \cite{point-based-value-iteration-an-anytime-algorithm-for-pomdps} for finding an $\varepsilon$-optimal POMDP policy $\pi_{\text{POMDP}}$. Let $B \subset \Delta(M)$ be a $\delta$-dense discretization of the belief space, where $\delta = \varepsilon \cdot (1 - p_{\max})^2$. Note that while the POMDP state space has size $2|M| + 1$, the effective belief space dimension is $|M|$ since the belief is always concentrated on either $\{m_s: m \in M\}$, $\{m_f: m \in M\}$ or $\bot$ states, without overlapping. The PBVI algorithm has complexity $O(|B|^2 \cdot |K| \cdot |M|)$ per iteration, and is guaranteed to return an $\varepsilon$-optimal policy. A known result in measure theory states that we can choose $B$ such that $|B| = O((1/\delta)^{|M| - 1})$. Thus, the total runtime over $H$ iterations is $O\left( H \cdot \left( \frac{1}{\varepsilon \cdot (1 - p_{\max})^2} \right)^{2|M| - 2} \cdot |K| \cdot |M| \right)$. Given $\pi_{\text{POMDP}}$, we extract a policy for $\prob$ from $\pi_{\text{POMDP}}$ by unrolling its decisions from the initial belief $b_0$, assuming we always receive positive feedback. Specifically, we construct $\pi = (k_1, \ldots, k_H)$ where $k_t$ is the action that $\pi_{\text{POMDP}}$ recommends at step $t$ when following this trajectory. When computing the expected value of $\pi_{\text{POMDP}}$ using the standard Bellman equation, we note that transitions to the absorbing state $\bot$ contribute zero value. Therefore, the expected value depends only on the trajectory that receives positive feedback throughout its duration. Since this trajectory prescribes the same sequence of actions as $\pi$, and the value structure matches $\prob$ (as established above), we have $V^{\pi}(q) = V_{\text{POMDP}}^{\pi_{\text{POMDP}}}(b_0)$.

Since $\pi_{\text{POMDP}}$ is $\varepsilon$-optimal for the POMDP, we have $V^{\pi}(q) = V_{\text{POMDP}}^{\pi_{\text{POMDP}}}(b_0) \geq V_{\text{POMDP}}^\star(b_0) - \varepsilon \geq \vs(q) - \varepsilon$, where the final inequality follows from the fact that the optimal policy $\ps$ in $\prob$ corresponds to a POMDP policy achieving the same value. Thus, $\pi$ is $\varepsilon$-optimal for $\prob$, completing the proof.
\end{proof}

\section{Convergence of Optimal Policies}
\label{sec:convergence-of-the-optimal-policy}
In \Cref{sec:approximating-the-optimal-policy}, we showed that finite horizon analysis suffices to approximate optimal policies. Here, we establish an even more fundamental property for a broad family of instances: After a finite number of rounds, the optimal policy becomes fixed, transitioning from exploration to pure exploitation. We formalize this result later on in this section through \Cref{thm:convergence}. This convergence introduces a natural complexity measure through the time it takes to converge. Through the proof's construction, we can identify both when convergence occurs and determine the optimal policy from that point onward. Hence, instances with faster convergence require less computational effort as fewer possibilities need to be explored before identifying the final repeating recommendation. While \Cref{thm:convergence} only establishes the convergence but doesn't provide a meaningful rate, our empirical analysis in \Cref{sec:experiments} reveals that this transition typically occurs relatively fast in practice.

The remainder of this section is devoted to formalizing and proving \Cref{thm:convergence}. We begin in Subsection~\ref{subsec:convergence-notations-and-class-of-instances} by defining the broad class of instances for which convergence is guaranteed and introducing the notations used throughout the analysis. In Subsections~\ref{subsec:convergence-analysis-of-unconcentrated-beliefs}--\ref{subsec:convergence-concentrated-beliefs}, we establish two auxiliary propositions that characterize key properties of the optimal policy. Finally, in Subsection~\ref{subsec:convergence-conclusion}, we synthesize these results to state and prove \Cref{thm:convergence}, and present a useful condition to determine when convergence occurs.

\subsection{Preliminaries: Separation and Belief Concentration}\label{subsec:convergence-notations-and-class-of-instances}

We begin by defining a useful parameter $c$, which we refer to as the \emph{separator} of the instance, and which underlies all the results in this section. This parameter captures key structural properties of the preference matrix $\bP$ that guarantee informative belief updates and eventual convergence of the optimal policy. Specifically, the separator quantifies two distinct forms of separation: (i) the degree to which different user types exhibit divergent preferences over the same items, and (ii) the extent to which each user type distinctly prefers its most liked item over all others. The first condition ensures that observed feedback can effectively distinguish between user types, while the second guarantees that, once the user type is sufficiently identified, a clear optimal recommendation exists.

\begin{definition}[Separator]
Given a Rec-APC instance $\mathcal{I} = \langle M, K, \bq, \bP \rangle$, we define its \emph{separator} $c(\mathcal{I})$ as the minimum of the following three quantities:
\begin{enumerate}
    \item\label{def:c-item-1} $\min_{k \in K,\, m \ne m' \in M} \abs{ \bP(k, m) - \bP(k, m') }$;
    \item\label{def:c-item-2} $\min_{m \in M} \left\{ \bP_{(1)}(\cdot, m) - \bP_{(2)}(\cdot, m) \right\}$;
    \item\label{def:c-item-3} $1 - p_{\max}$, where $p_{\max} := \max_{k,m} \bP(k,m)$;
\end{enumerate}
Here, $\bP_{(1)}(\cdot, m)$ and $\bP_{(2)}(\cdot, m)$ denote the largest and second-largest entries in column $m$ of $\bP$, respectively.
\end{definition}

Item~\eqref{def:c-item-1} serves as a measure of heterogeneity across user types, as it quantifies how users of distinct types respond to the same category. Item~\eqref{def:c-item-2} captures the separation between the most and second-most preferred categories for each user type, ensuring that the best option is well-defined. Finally, Item~\eqref{def:c-item-3} captures how close the maximum entry of $\bP$ is to 1, thereby controlling the extremity of user preferences.

\begin{definition}[Well-separated Instance]\label{def:well separated}
We say that an instance $\mathcal{I}$ is \emph{well-separated} if $c(\mathcal{I}) > 0$.   
\end{definition}
We further let $\mathcal F$ denote the class of all well-separated instances, namely
\[
\mathcal{F} := \left\{ \mathcal{I} \mid c(\mathcal{I}) > 0 \right\}.
\]
This is a broad class that includes, but is not limited to, all instances where the matrix $\bP$ has pairwise distinct entries. The condition $c(\mathcal{I}) > 0$ guarantees that belief updates are informative and that optimal policies eventually converge to recommending a fixed category. In contrast, when $c(\mathcal{I}) = 0$, such convergence is not assured. For example, consider an instance where all the rows in $\bP$ are identical. Then, as any policy will yield the same outcome, convergence is not guaranteed.  When the instance is clear from context, we write $c$ instead of $c(\mathcal{I})$.

To make our subsequent analysis more precise, we quantify the concentration of a belief on a single user type.
\begin{definition}[$(\delta,m)$-concentration]
    Let $m \in M$ and $\delta > 0$. A belief $\bb \in \Delta(M)$ is \emph{$(\delta,m)$-concentrated} if $\bb(m) > 1 - \delta$. 
\end{definition}
We will also need the complementary notion of an unconcentrated belief, which we refer to as $\delta$-unconcentrated.
\begin{definition}[\emph{$\delta$-unconcentratation}]
    Let $\delta > 0$. A belief $\bb \in \Delta(M)$ is \emph{$\delta$-unconcentrated} if for all $m \in M$, $\bb(m) < 1 - \delta$.    
\end{definition}
The distinction between concentrated and unconcentrated beliefs is central to our analysis. When a belief is sufficiently concentrated--i.e., when $\delta$ is small enough--the myopic action is also optimal (as formalized in \Cref{prop:myopic-near-boundary}). We refer to such regions of the belief space as \emph{concentrated regions}. However, in these regions, the informativeness of feedback is limited since the belief is already highly concentrated. In contrast, in \emph{unconcentrated regions}, where uncertainty over the user type remains substantial, the optimal policy may deviate from myopic behavior to strategically acquire information through highly informative feedback.

\subsection{Analysis of Unconcentrated Beliefs}\label{subsec:convergence-analysis-of-unconcentrated-beliefs}

In this subsection, we examine the behavior of the optimal policy in unconcentrated regions of the belief space. We begin with \Cref{prop:gap-between-value-function}, which establishes that the optimal value function $\vs$ exhibits a strictly positive gap between consecutive beliefs along the belief walk induced by the optimal policy.

\begin{proposition}\label{prop:gap-between-value-function}
   For any $\delta$-unconcentrated belief $\bb \in \Delta(M)$ it holds that $\vs(\tau(\bb, \ps_1(\bb))) - \vs(\bb) \geq \frac{\delta \cdot (1 - \delta) \cdot c^2}{1 - c}$.
\end{proposition}

\Cref{prop:gap-between-value-function} partially addresses the trade-off between immediate and future rewards highlighted in \Cref{obs:recursive-formula-of-the-value-function}. While the optimal policy does not necessarily choose the category that maximizes the next-step value, it guarantees an improvement over the current value. That is, it avoids transitions to beliefs with strictly lower value, even when such transitions yield high immediate rewards.

\begin{proof}[\proofof{prop:gap-between-value-function}]
Recall from \Cref{obs:recursive-formula-of-the-value-function} that
\[
\vs(\bb) = p_k(\bb) \cdot (1 + \vs(\tau(\bb, k))) \quad \text{for } k := \ps_1(\bb),
\]
which implies
\[
\vs(\tau(\bb, k)) - \vs(\bb) = \left( \frac{\vs(\bb)}{p_k(\bb)} - 1 \right) - \vs(\bb) = \frac{\vs(\bb)(1 - p_k(\bb)) - p_k(\bb)}{p_k(\bb)}.
\]
Let $\hat{\pi}$ be the stationary policy that repeatedly recommends $k$. Since $\vs(\bb) \geq V^{\hat{\pi}}(\bb)$ and $1 - p_k(\bb) \geq 0$, we obtain the inequality:
\[
\vs(\tau(\bb, k)) - \vs(\bb) 
\geq \frac{V^{\hat{\pi}}(\bb)(1 - p_k(\bb)) - p_k(\bb)}{p_k(\bb)} 
= V^{\hat{\pi}}(\tau(\bb, k)) - V^{\hat{\pi}}(\bb).
\]
Thus, bounding $V^{\hat{\pi}}(\tau(\bb, k)) - V^{\hat{\pi}}(\bb)$ will be sufficient for for proving the proposition. Let $\tau^i$ be the $(i+1)$-th belief in $\left(\bb^{\hat \pi, \bb}\right)$. Namely, we have $\tau^1 = \tau(\bb, k)$ and $\tau^2 =\tau(\tau(\bb, k), k)$, and so on. Using \Cref{obs:recursive-formula-of-the-value-function}, we get
\[
    V^{\hat{\pi}}(\tau(\bb, k)) - V^{\hat{\pi}}(\bb) 
    = p_k(\tau(\bb, k)) \cdot (1 + V^{\hat{\pi}}(\tau^2)) - p_k(\bb) \cdot (1 + V^{\hat{\pi}}(\tau^1)).
\]
\Cref{lemma:closed-form-representations-of-the-value-function} suggests that $V^{\hat{\pi}}(\tau^1) = \sum_{t=1}^{\infty} \prod_{j=1}^{t} p_{k}(\tau^t)$, and similarly, $V^{\hat{\pi}}(\tau^2) = \sum_{t=2}^{\infty} \prod_{j=1}^{t} p_{k}(\tau^t)$. Thus, showing that $p_k(\tau^{t+1}) \geq p_k(\tau^{t})$ for all $t$ would imply $V^{\hat{\pi}}(\tau^2) \geq V^{\hat{\pi}}(\tau^1)$, and we could obtain the following:

\begin{align}\label{ineq:value-difference-bigger-than-prob-difference}
V^{\hat{\pi}}(\tau(\bb, k)) - V^{\hat{\pi}}(\bb) &= p_k(\tau(\bb, k)) \cdot (1 + V^{\hat{\pi}}(\tau^2)) - p_k(\bb) \cdot (1 + V^{\hat{\pi}}(\tau^1)) \nonumber \\
&\geq \left(p_k(\tau(\bb, k)) - p_k(\bb)\right) \cdot (1 + V^{\hat{\pi}}(\tau^1)) \nonumber \\
&\geq p_k(\tau(\bb, k)) - p_k(\bb).
\end{align}

For the rest of the proof, we focus on lower bounding the difference $p_k(\tau(\bb, k)) - p_k(\bb)$, which is the right-hand side of Inequality~\eqref{ineq:value-difference-bigger-than-prob-difference}, as a function of $\delta$ and $c$. The arguments we present below hold for every $\delta \geq 0$ and for every $\delta$-unconcentrated belief, which ensures that $p_k(\tau^{t+1}) \geq p_k(\tau^{t})$ holds true. We start by expanding the difference $p_k(\tau(\bb, k)) - p_k(\bb)$:
\begin{align}
p_k(\tau(\bb, k)) - p_k(\bb) 
&= \sum_{m} \bP(k, m) \cdot \tau(\bb, k)(m) - \sum_{m} \bP(k, m) \cdot \bb(m) \notag \\
&= \sum_{m} \bP(k, m) \cdot \frac{\bP(k, m) \cdot \bb(m)}{\sum_{m'} \bP(k, m') \cdot \bb(m')} - \sum_{m} \bP(k, m) \cdot \bb(m) \notag \\
&= \frac{\sum_{m} \bP(k, m)^2 \cdot \bb(m) - \left( \sum_{m} \bP(k, m) \cdot \bb(m)\right)^2}{\sum_{m} \bP(k, m) \cdot \bb(m)}. \label{eq:gap-proof-prob-difference-variance}
\end{align}

Note that the numerator of Equation~\eqref{eq:gap-proof-prob-difference-variance} can be rewritten as:
\begin{equation}\label{eq:gap-proof-numerator}
    \sum_{m} \bP(k, m)^2 \cdot \bb(m) - \left( \sum_{m} \bP(k, m) \cdot \bb(m)\right)^2 = \sum_{m < m'} \bb(m)\bb(m')(\bP(k,m) - \bP(k,m'))^2.
\end{equation}

By the definition of $c$, for any distinct $m, m' \in M$, we have $|\bP(k,m) - \bP(k,m')| \geq c$. Therefore, we can lower bound the right-hand side of \Cref{eq:gap-proof-numerator} as follows:
\begin{equation}\label{eq:gap-proof-numerator-bounding}
    \sum_{m < m'} \bb(m)\bb(m')(\bP(k,m) - \bP(k,m'))^2 \geq c^2 \sum_{m < m'} \bb(m)\bb(m') = \frac{c^2}{2} \left( 1 - \sum_m \bb(m)^2 \right).
\end{equation}

Since $\bb$ is $\delta$-unconcentrated, we have $\bb(m) < 1 - \delta$ for all $m \in M$. The maximum value of $\sum_m \bb(m)^2$ subject to this constraint and $\sum_m \bb(m) = 1$ is achieved when $\bb(m_1) = 1 - \delta$ and $\bb(m_2) = \delta$ (with all other coordinates zero). Using this to further simplify Inequality~\eqref{eq:gap-proof-numerator-bounding}, we get:
\begin{equation}\label{eq:gap-proof-numerator-bounding-only-with-c-delta}
    \sum_{m < m'} \bb(m)\bb(m')(\bP(k,m) - \bP(k,m'))^2 \geq \frac{c^2}{2} \left( 1 - \left((1-\delta)^2 + \delta^2\right) \right) = \delta(1-\delta) c^2.
\end{equation}

Finally, since $\sum_{m} \bP(k, m) \cdot \bb(m) \leq 1 - c$, plugging Inequality~\eqref{eq:gap-proof-numerator-bounding-only-with-c-delta} into \Cref{eq:gap-proof-prob-difference-variance} yields:
\[ 
\vs(\tau(\bb, k)) - \vs(\bb) \geq p_k(\tau(\bb, k)) - p_k(\bb) \geq \frac{\delta(1-\delta) c^2}{1-c}, 
\]
and the proof of \Cref{prop:gap-between-value-function} is complete.
\end{proof}

As an immediate consequence of \Cref{prop:gap-between-value-function}, we can establish that the number of unconcentrated beliefs visited along the belief walk induced by the optimal policy is bounded. This follows from a standard potential function argument: Each visit to a $\delta$-unconcentrated belief yields a strictly positive improvement in the value function. Since the value function is bounded from above, as implied by \Cref{lemma:closed-form-representations-of-the-value-function}, the number of such improvements must be finite. We formulate this intuition in the following corollary.

\begin{corollary}\label{cor:limited-unconcetrated}
For a fixed $\delta > 0$, the optimal belief walk initiating from prior $\bq$ can contain at most $H = \left\lceil\frac{(1 - c)^2}{\delta \cdot (1 - \delta) \cdot c^3}\right\rceil$ $\delta$-unconcentrated beliefs.
\end{corollary}

\begin{proof}[\proofof{cor:limited-unconcetrated}]
    First, notice that \Cref{eq:bounded-value} ensures that the value function is upper bounded by $\frac{p_{\max}}{1-p_{\max}} \leq \frac{1-c}{c}$. Now, consider the belief walk $(\bb^{\ps, \bq}_t)_{t=1}^{\infty}$. Let $S$ be the sequence of indices $t$ where $\bb^{\ps, \bq}_t$ is $\delta$-unconcentrated. Assume for contradiction that $|S| > H$, and denote $r$ to be the $H+1$'th index in $S$. By \Cref{prop:gap-between-value-function}, each unconcentrated belief in the walk contributes an increase of at least $\frac{\delta \cdot (1 - \delta) \cdot c^2}{(1 - c)}$ to the optimal value function. Therefore, having more than $H$ unconcentrated beliefs would imply that $\vs(\bb^{\ps, \bq}_{r+1}) > \frac{1-c}{c}$, contradicting the upper bound on the value function established in \Cref{eq:bounded-value}.
\end{proof}

\subsection{Characterizing the Optimal Policy at Concentrated Beliefs}\label{subsec:convergence-concentrated-beliefs}

Having established that the belief walk can only visit finitely many unconcentrated beliefs, we now turn our attention to characterizing the optimal policy's behavior in concentrated regions of the belief space. We first establish a useful regularity property of the value function.

\begin{lemma}
    \label{lemma:lipschitz}
    The optimal value function $\vs$ is Lipschitz continuous with constant $L = \frac{1-c}{c}$.
\end{lemma}

\begin{proof}[\proofof{lemma:lipschitz}]
Let $\bb, \bb' \in \Delta(M)$ be two beliefs. Using $|a-b| = \max\{a-b, b-a\}$ and the optimality of $\ps(\bb)$ for $\bb$ and $\ps(\bb')$ for $\bb'$:
\begin{align*}
|\vs(\bb) - \vs(\bb')| &= \max\{\vs(\bb) - \vs(\bb'), \vs(\bb') - \vs(\bb)\} \\
&\leq \max\{V^{\ps(\bb)}(\bb) - V^{\ps(\bb)}(\bb'), V^{\ps(\bb')}(\bb') - V^{\ps(\bb')}(\bb)\}
\end{align*}
By \Cref{lemma:closed-form-representations-of-the-value-function}, $V^{\pi}(\bb) = \sum_{m \in M} \bb(m) \cdot \sum_{t=1}^{\infty} \prod_{i=1}^{t} \bP(\pi_i, m)$, so:
\begin{align*}
&= \max\left\{\sum_{m \in M} (\bb(m) - \bb'(m)) \cdot \sum_{t=1}^{\infty} \prod_{i=1}^{t} \bP(\ps_i(\bb), m), \quad \text{symmetric term}\right\} \\
&\leq \sum_{m \in M} |\bb(m) - \bb'(m)| \cdot \sum_{t=1}^{\infty} \prod_{i=1}^{t} p_{\max} = \norm{\bb - \bb'}_1 \cdot \frac{p_{\max}}{1-p_{\max}} \leq \norm{\bb - \bb'}_1 \cdot \frac{1-c}{c}.
\end{align*}
\end{proof}

The next proposition characterizes the optimal policy in concentrated regions of the belief space. It establishes an intuitive result: If the belief is sufficiently concentrated -- that is, the system is almost certain about the user’s type -- then the optimal policy is myopic; that is, it selects the category most preferred by the most likely user type.

\begin{proposition}\label{prop:myopic-near-boundary}
   For every type $m \in M$ and $(\frac{c^2}{4}, m)$-concentrated belief $\bb$, $\ps_1(\bb) = \argmax_{k \in K} \bP(k, m)$.
\end{proposition}

Although \Cref{prop:myopic-near-boundary} appears natural, its proof requires careful analysis. While it is straightforward to verify that the myopic recommendation is indeed $\argmax_{k \in K} \bP(k, m)$, the main challenge lies in showing that this recommendation remains optimal even when future recommendations are taken into account. The proof is somewhat technical, but the underlying geometric intuition is insightful. Consider a $(\delta, m)$-concentrated belief $\bb$ for a small $\delta$. Because $\bb$ is concentrated, the immediate reward gap between the category $k = \argmax_{k'' \in K} \bP(k'', m)$ and any other category $k' \neq k$ is significant, due to a difference close to $c$ in their corresponding "like" probabilities. At the same time, because the belief is highly concentrated, the posterior distributions $\tau(\bb, k)$ and $\tau(\bb, k')$ remain close to the $\bb$ and to each other. By the Lipschitz continuity of the value function (\Cref{lemma:lipschitz}), their respective future values are therefore nearly identical. Thus, in the recursive expression from \Cref{obs:recursive-formula-of-the-value-function}, the important term is the immediate reward, as the future rewards are similar regardless of the recommendation, and the optimal action is to act myopically.

\begin{proof}[\proofof{prop:myopic-near-boundary}]
Let $m \in M$ be a user type, $k = \argmax_{k' \in K} \bP(k', m)$ and $\bb \in \Delta(M)$ such that $\bb(m) \geq 1 - \frac{c^2}{4}$. Denote $\pi^{k} = (k)_{i=1}^{\infty}$, the policy that always recommends category $k$. Let $\hat{\pi}$ be an arbitrary policy such that $\hat{\pi}(\bb)_1 = \hat{k}$ for some $\hat{k} \neq k$. Our goal is to demonstrate that $V^{\pi^k}(\bb) \geq V^{\hat{\pi}}(\bb)$, which ensures that the first recommendation of the optimal policy is $k$.

\Cref{lemma:closed-form-representations-of-the-value-function} suggests $V^{\pi^k}(\bb) = \sum_{m' \in M} \bb(m') \cdot \frac{\bP(k, m')}{1 - \bP(k, m')}$. Since $\bb(m) \geq 1 - \frac{c^2}{4}$, we can bound this expression from below:
\[
V^{\pi^k}(\bb) \geq \bb(m) \cdot \frac{\bP(k, m)}{1 - \bP(k, m)} \geq \left( 1 - \frac{c^2}{4} \right) \cdot \frac{\bP(k, m)}{1 - \bP(k, m)}.
\]
For the competing policy, \Cref{obs:recursive-formula-of-the-value-function} gives us:
\begin{equation}\label{eq:policy-starts-suboptimal-myopic-proof}
    V^{\hat{\pi}}(\bb) = \left( \bb(m) \cdot \bP(\hat{k}, m) + \sum_{m' \neq m} \bb(m') \cdot \bP(\hat{k}, m') \right) \cdot \left( 1 + V^{\hat{\pi}[2:]}(\tau(\bb, \hat{k})) \right).
\end{equation}
Denote $\bm e_m$ as the unit vector that corresponds to user type $m$. By employing the Lipschitz property established in \Cref{lemma:lipschitz} on the beliefs $\tau(\bb, \hat{k})$ and $\bm e_m$, we obtain:
\begin{equation}\label{eq:lipschitz-bound-myopic-proof}
    V^{\hat{\pi}[2:]}(\tau(\bb, \hat{k})) \leq \vs(\tau(\bb, \hat{k})) \leq \vs(\mathbf{e}_m) + \norm{\tau(\bb, \hat{k}) - \mathbf{e}_m}_1 \cdot \frac{1-c}{c}.
\end{equation}
Since $\mathbf{e}_m$ represents certainty that the user is of type $m$, the optimal policy at this belief is to always recommend category $k$ (the most preferred category for type $m$), yielding $\vs(\mathbf{e}_m) = \frac{\bP(k, m)}{1 - \bP(k, m)}$. Additionally, the distance $\norm{\tau(\bb, \hat{k}) - \mathbf{e}_m}_1$ simplifies as follows:
\[
\begin{aligned}
   \norm{\tau(\bb, \hat{k}) - \mathbf{e}_m}_1 &= \sum_{m' \in M} |\tau(\bb, \hat{k})(m') - \mathbf{e}_m(m')| \\
   &= \left( 1 - \frac{\bb(m) \cdot \bP(\hat{k}, m)}{\bb(m) \cdot \bP(\hat{k}, m) + \sum_{m' \neq m} \bb(m') \cdot \bP(\hat{k}, m')} \right) \\
   &+ \frac{\sum_{m' \neq m} \bb(m') \cdot \bP(\hat{k}, m')}{\bb(m) \cdot \bP(\hat{k}, m) + \sum_{m' \neq m} \bb(m') \cdot \bP(\hat{k}, m')} \\
   &= \frac{2 \cdot \sum_{m' \neq m} \bb(m') \cdot \bP(\hat{k}, m')}{\bb(m) \cdot \bP(\hat{k}, m) + \sum_{m' \neq m} \bb(m') \cdot \bP(\hat{k}, m')}.
\end{aligned}
\]
Plugging the above values back into Inequality~\eqref{eq:lipschitz-bound-myopic-proof}, we obtain:
\[
    1 + V^{\hat{\pi}[2:]}(\tau(\bb, \hat{k})) \leq 1 + \frac{\bP(k, m)}{1 - \bP(k, m)} + \left(\frac{2 \cdot \sum_{m' \neq m} \bb(m') \cdot \bP(\hat{k}, m')}{\bb(m) \cdot \bP(\hat{k}, m) + \sum_{m' \neq m} \bb(m') \cdot \bP(\hat{k}, m')}\right) \cdot \frac{1-c}{c}.
\]
Therefore, by substituting this upper bound for $1 + V^{\hat{\pi}[2:]}(\tau(\bb, \hat{k}))$ into \Cref{eq:policy-starts-suboptimal-myopic-proof}, it holds that:
\[
V^{\hat{\pi}}(\bb) \leq \left( \bb(m) \cdot \bP(\hat{k}, m) + \sum_{m' \neq m} \bb(m') \cdot \bP(\hat{k}, m') \right) \cdot \frac{1}{1 - \bP(k, m)} + 2 \sum_{m' \neq m} \bb(m') \cdot \bP(\hat{k}, m') \cdot \frac{1-c}{c}.
\]
Using the fact that $\bb$ is $(\frac{c^2}{4}, m)$-concentrated -- that is, $\sum_{m' \neq m} \bb(m') \leq \frac{c^2}{4}$ -- and $\bP(\hat{k}, m) \leq \bP(k, m) - c$ from the definition of $c$, we derive an upper bound on $V^{\hat{\pi}}(\bb)$:
\[
V^{\hat{\pi}}(\bb) \leq \left(\bP(k, m) - c + \frac{c^2}{4} \right) \cdot \frac{1}{1 - \bP(k, m)} + \frac{c^2}{2} \cdot \frac{1-c}{c}.
\]
To show that $V^{\pi^k}(\bb) \geq V^{\hat{\pi}}(\bb)$, it suffices to prove:
\begin{equation}\label{eq:myopic-proof-temp-final-bounds-comparision}
\left( 1 - \frac{c^2}{4} \right) \cdot \frac{\bP(k, m)}{1 - \bP(k, m)} > \left(\bP(k, m) - c + \frac{c^2}{4} \right) \cdot \frac{1}{1 - \bP(k, m)} + \frac{c(1-c)}{2}.
\end{equation}
After rearrangement, Inequality~\eqref{eq:myopic-proof-temp-final-bounds-comparision} simplifies to:
\[
\frac{c}{2} + \frac{c^2}{4} > \bP(k, m)\left(\frac{3c^2}{4} - \frac{c}{2}\right).
\]
Since $0 \leq \bP(k, m) < 1$ and $0 < c \leq 1$, this inequality holds, and the proof is complete.
\end{proof}

\subsection{Convergence of Optimal Policies in Well-Separated Instances}\label{subsec:convergence-conclusion}

We are now equipped with all the necessary tools to establish our main convergence result.
\begin{theorem}\label{thm:convergence}
   Fix any well-separated instance $\mathcal I \in \mathcal F$. There exists $T < \infty$ such that for any $t \geq T$, it holds that $\ps_{t+1}=\ps_{t}$.
\end{theorem}
The proof uses both \Cref{cor:limited-unconcetrated} with \Cref{prop:myopic-near-boundary}, while addressing a subtle technical issue related to transitions between different concentrated regions.

\begin{proof}[\proofof{thm:convergence}]
Recall that \Cref{cor:limited-unconcetrated} guarantees that the optimal belief walk can only visit finitely many unconcentrated beliefs. Therefore, after a finite time, the belief walk must enter a concentrated region. In such regions, \Cref{prop:myopic-near-boundary} ensures that the optimal policy is myopic; thus, given that the belief walk stays in the same concentrated region, the optimal policy becomes fixed. 

However, we still have to address a potential edge case: Could the belief walk alternate between different concentrated regions over time? If the walk could switch from being $(\frac{c^2}{4},m)$-concentrated to $(\frac{c^2}{4},m')$-concentrated for distinct types $m \neq m'$, then the optimal recommendations might differ between regions, violating convergence. To resolve this, we now show that any transition between different concentrated regions must pass through an unconcentrated belief. Consider any $(\frac{c^2}{4},m)$-concentrated belief $\bb$ and let $\bb' = \tau(\bb, \ps_1(\bb))$ be the updated belief after the optimal recommendation. For any distinct type $m' \neq m$, we have $\bb(m') < \frac{c^2}{4}$ since $\bb(m) \geq 1 - \frac{c^2}{4}$. By the Bayesian update formula:
\[
\bb'(m') = \frac{\bb(m') \cdot \bP(\ps_1(\bb), m')}{\sum_{m'' \in M} \bb(m'') \cdot \bP(\ps_1(\bb), m'')} \leq \frac{\frac{c^2}{4} \cdot 1}{(1 - \frac{c^2}{4}) \cdot c} = \frac{c}{4 - c^2} \leq \frac{1}{3} < 1 - \frac{c^2}{4},
\]
where the penultimate inequality uses the fact that $\frac{c}{4-c^2}$ is maximized at $c=1$ giving $\frac{1}{3}$, and the final inequality holds since $1 - \frac{c^2}{4} \geq \frac{3}{4}$ for $c \in [0,1]$. Therefore, $\bb'$ cannot be $(\frac{c^2}{4},m')$-concentrated for any $m' \neq m$. This means that any transition between distinct concentrated regions must pass through an unconcentrated belief. Since we established that only finitely many unconcentrated beliefs can be visited, only finitely many such transitions can occur.

Combining these results: The belief walk eventually remains concentrated around a single type $m$ after finite time $T$, and from that point forward the optimal policy repeatedly recommends $\argmax_{k \in K} \bP(k,m)$. Thus, $\ps_{t+1} = \ps_t$ for all $t \geq T$, establishing convergence.
\end{proof}

\Cref{thm:convergence} shows that the optimal policy eventually converges, as the belief walk becomes concentrated around a single user type. However, in the general case, we cannot determine in advance which type the belief will converge to. Moreover, as discussed in the proof of \Cref{thm:convergence}, the belief may pass near multiple user types before settling. Identifying the point of convergence offers computational benefits: Once convergence is detected, the search can be terminated, since the optimal action remains fixed from that point onward.

We now characterize when such convergence occurs, based on the structure of the belief and the user type around which it is concentrated. 

\begin{proposition}\label{prop:convergence-condition}
    Let $m \in M$ be a user type and let $\bb \in \Delta(M)$ be a belief that is $(\frac{c^2}{4}, m)$-concentrated. Let $k = \argmax_{k' \in K} \bP(k', m)$. Then the belief walk induced by the optimal policy will remain $(\frac{c^2}{4}, m)$-concentrated indefinitely if and only if 
    \begin{equation}\label{eq:prop-convergence-condition}
    m \in \argmax_{m' \in \mathrm{supp}(\bb)} \bP(k, m'), \quad \text{where }\mathrm{supp}(\bb) = \{ m' \in M : \bb(m') > 0 \}.
    \end{equation}
\end{proposition}

\begin{proof}[\proofof{prop:convergence-condition}]
    The proof consists of two parts: Sufficiency and necessity.
    
    \paragraph{($\Rightarrow$) Sufficiency.} Suppose $m$ satisfies Condition~\eqref{eq:prop-convergence-condition}. Then for all $m' \in \mathrm{supp}(\bb)$, it holds that $\bP(k, m') \le \bP(k, m)$. Therefore, the belief update following a recommendation of $k$ satisfies
    \[
        \tau(\bb, k)(m) = \frac{\bb(m) \bP(k, m)}{\sum_{m' \in M} \bb(m') \bP(k, m')} \ge \bb(m),
    \]
    Thus, the posterior maintains or increases belief in $m$, and $\tau(\bb, k)$ remains $(\delta, m)$-concentrated for some $\delta \leq \frac{c^2}{4}$. By \Cref{prop:myopic-near-boundary}, the optimal policy continues to recommend $k$, and the belief remains concentrated near $m$ indefinitely.

    \paragraph{($\Leftarrow$) Necessity.} Now suppose $m$ does not satisfy Condition~\eqref{eq:prop-convergence-condition}. Then there exists $m' \in \mathrm{supp}(\bb)$ such that $\bP(k, m') > \bP(k, m)$. Consider the ratio of post-update beliefs between $m'$ and $m$:
    \begin{equation}\label{eq:relative-mass}
        \frac{\tau(\bb, k)(m')}{\tau(\bb, k)(m)} = \frac{\frac{\bb(m') \bP(k, m')}{\sum_{m'' \in M} \bb(m'') \bP(k, m'')}}{\frac{\bb(m) \bP(k, m)}{\sum_{m'' \in M} \bb(m'') \bP(k, m)}} = \frac{\bb(m') \bP(k, m')}{\bb(m) \bP(k, m)}.
    \end{equation}
    Since $\bP(k, m') > \bP(k, m) + c$ from the definition of $c$, we can bound the right-hand side of Inequality~\eqref{eq:relative-mass} as follows:
    \begin{equation}\label{eq:relative-mass-increase-bound}
        \frac{\tau(\bb, k)(m')}{\tau(\bb, k)(m)} \geq \frac{\bb(m')}{\bb(m)} \cdot \frac{\bP(k, m) + c}{\bP(k, m)} \geq \frac{\bb(m')}{\bb(m)} \cdot \left( 1 + \frac{c}{\bP(k, m)} \right) \geq \frac{\bb(m')}{\bb(m)} \cdot (1+c).
    \end{equation}
    Now, assume towards contradiction that the belief walk remains $(\frac{c^2}{4}, m)$-concentrated indefinitely. By \Cref{prop:myopic-near-boundary}, the optimal policy will therefore always recommend $k$. Denote the belief after $N$ recommendations of $k$ as $\tau(\bb, k)^N$. Then, the ratio of post-update beliefs after $N$ recommendations is bounded as follows:
    \begin{equation}\label{eq:relative-mass-increase-induction}
        \frac{\tau(\bb, k)^N(m')}{\tau(\bb, k)^N(m)} \geq \frac{\tau(\bb, k)^{N-1}(m')}{\tau(\bb, k)^{N-1}(m)} \cdot (1+c) \geq \ldots \geq \frac{\tau(\bb, k)(m')}{\tau(\bb, k)(m)} \cdot (1+c)^{N-1} \geq \frac{\bb(m')}{\bb(m)} \cdot (1+c)^N.
    \end{equation}
    Inequality~\eqref{eq:relative-mass-increase-induction} shows that the ratio of post-update beliefs between $m'$ and $m$ grows exponentially with $N$ as long as $k$ is recommended. However, since we assume $\tau(\bb, k)^N$ is $(\frac{c^2}{4}, m)$-concentrated for all $N$, it holds that $\tau(\bb, k)^N(m') \leq \frac{c^2}{4}$ and $\tau(\bb, k)^N(m) \geq 1 - \frac{c^2}{4}$. Therefore, the ratio $\frac{\tau(\bb, k)^N(m')}{\tau(\bb, k)^N(m)}$ is bounded above by $\frac{c^2/4}{1-c^2/4} = \frac{c^2}{4 - c^2}$. This leads to a contradiction with Inequality~\eqref{eq:relative-mass-increase-induction} as $N$ grows large enough, which implies that the belief walk cannot remain $(\frac{c^2}{4}, m)$-concentrated indefinitely.
\end{proof}

\section{A Branch-and-Bound Algorithm}
\label{sec:branch-and-bound-algorithm}

In this section, we introduce a branch-and-bound (B\&B) algorithm tailored to our setting, outlined in \Cref{bb-algorithm}. The B\&B approach is widely used in sequential decision-making and combinatorial optimization~\citep{learning-to-branch-with-tree-mdps,reinforcement-learning-for-branch-and-bound-optimisation-using-retrospective-trajectories,learning-to-search-in-branch-and-bound-algorithms}. Its effectiveness hinges on the quality of the bounds used to evaluate the search space and eliminate suboptimal policies. For our problem, we derive these bounds based on the recursive structure of the value function $V^\pi$. Specifically, for any policy $\pi$ and positive integer $h \in \mathbb{N}$, the value function $V^\pi$ can be expressed as:

\begin{equation}
  \label{eq:value-function-with-prefix-considerations}
  V^\pi = \sum_{t=1}^{h} \prod_{i=1}^{t} p_{\pi_i}(\bb^{\pi, \bq}_i) + \prod_{i=1}^{h} p_{\pi_i}(\bb^{\pi, \bq}_{i}) \cdot V^{\pi[h+1:]}(\bb^{\pi, \bq}_{h+1}).
\end{equation}

\Cref{eq:value-function-with-prefix-considerations} decomposes $V^\pi$ into two components: The cumulative rewards for the first $h$ rounds and a recursive term representing the discounted expected value of future rounds. Crucially, replacing $V^{\pi[h+1:]}(\bb^{\pi, \bq}_{h+1})$ with an upper or lower bound yields the corresponding bounds for $V^\pi$.

Building on this, we now propose an upper bound. Intuitively, imagine that the RS is entirely certain about the user type. That is, the type would still be sampled according to the belief $\bb$, but the RS could pick a policy conditioned on the sampled type. In such a case, the RS would pick the type's favorite category indefinitely, leading to an expected social welfare of:
\[
  V^U(\bb) := \sum_{m \in M} \bb(m) \cdot \max_{k \in K} \frac{\bP(k, m)}{1 - \bP(k, m)}.
\]
Since $\bP(\ps_j, m) \leq \max_{k} \bP(k, m)$ always holds, we have that
\[
    \vs(\bb) = \sum_{m \in M} \bb(m) \sum_{t=1}^{\infty} \prod_{j=1}^{t} \bP(\ps_j, m) \leq \sum_{m \in M} \bb(m)  \sum_{t=1}^{\infty} \prod_{j=1}^{t} \left( \max_{k} \bP(k, m) \right) = V^{U}(\bb).
\]
We stress that the above upper bound is not necessarily attainable. As for the lower bound, we compute the value of the best fixed-action policy w.r.t. the belief $\bb$, namely,
\[
  V^L(\bb) := \max_{k \in K} \sum_{m \in M} \bb(m) \cdot \frac{\bP(k, m)}{1 - \bP(k, m)}.
\]
Note that the lower bound \emph{is attainable} as it is the value of a valid policy (the best fixed-action policy). More specifically, each of the expressions $\sum_{m \in M} \bb(m) \cdot \frac{\bP(k, m)}{1 - \bP(k, m)}$ corresponds to the expected value of the fixed policy $\pi_k = \left( k \right)_{t=1}^{\infty}$, and the lower bound is taken to be the value of the best such policy. The validity of this bound follows directly from the definition of the optimal policy, which must yield at least the value of the best fixed-action policy.

For ease of notation, for any prefix $\Pi \in \bigcup_{h=1}^{\infty} K^h$ we denote by $\ovv_\Pi, \unv_\Pi$ the substitution of $V^U, V^L$ into \Cref{eq:value-function-with-prefix-considerations}, respectively. Using this notation, $\ovv_\Pi$ acts as an upper bound for the value of any policy that starts with the prefix $\Pi$. Additionally, $\unv_\Pi$ serves as a lower bound of the maximal value of all policies that begin with $\Pi$; namely, $\unv_\Pi \leq \max_{\pi:\text{ begins with }\Pi} V^\pi$. 
\begin{algorithm}[t]
  \caption{B\&B Algorithm for $\prob$}
  \label{bb-algorithm}
  \begin{algorithmic}[1]
    \REQUIRE Instance $\langle M, K, \bq, \bP \rangle$, error term $\varepsilon > 0$
    \ENSURE $\varepsilon$-approximate policy and its value
    \STATE $\tilde \Pi \gets \varnothing$ \COMMENT{The empty prefix} \label{bnbalg:empty_prefix}
    \STATE $\tilde V \gets V^{L}(\bq)$ \COMMENT{Lower bound of the empty prefix}
    \STATE $L \gets \text{empty queue}$
    \STATE Append $\tilde \Pi$ to $L$
    \WHILE{$L \neq \emptyset$} 
    \STATE Pop a prefix $\Pi$ from $L$
    \IFTHEN{$\tilde V < \unv_{\Pi}$}{$\tilde V \gets \unv_{\Pi}, \tilde \Pi \gets \Pi$ \alglinelabel{bnbalg:refine}} 
    \FOR{$k \in K$} 
    \IFTHEN{$\ovv_{\Pi \oplus k}-\tilde V > \varepsilon$}{Append $\Pi \oplus k$ to $L$} \alglinelabel{bnbalg:branching}
    \ENDFOR
    \ENDWHILE
    \STATE \textbf{Return} $\tilde \Pi$, $\tilde V$
  \end{algorithmic}
\end{algorithm}

We are ready to present \Cref{bb-algorithm}. The algorithm takes as input an instance and an error term $\varepsilon$, outputting an $\varepsilon$-approximately optimal policy along with its corresponding value. It maintains two key variables: The current best prefix $\tilde \Pi$ and its corresponding value $\tilde V$. Using a queue $L$ to systematically explore policy prefixes, the algorithm implements two critical operations. In Line~\ref{bnbalg:refine}, it performs value refinement by comparing $\tilde V$ against the lower bound $\unv_{\Pi}$ of the examined prefix $\Pi$, updating both $\tilde V$ and $\tilde \Pi$ when an improvement is found. Then, in Line~\ref{bnbalg:branching}, it considers all possible one-step branching of $\Pi$ by appending a category $k$. For each extended prefix $\Pi \oplus k$, it calculates its upper bound $\ovv_{\Pi \oplus k}$. If the potential improvement $\ovv_{\Pi \oplus k} - \tilde{V}$ is more significant than $\varepsilon$, the extended prefix is added to the queue for further exploration. Otherwise, the branch is pruned as it cannot lead to a better solution within the desired accuracy.

The next theorem ensures the $\varepsilon$-optimality of \Cref{bb-algorithm}.

\begin{theorem}\label{thm:bb-algorithm-bounded-error}
   For any input $\langle M, K, \bq, \bP \rangle$, $\varepsilon$, \Cref{bb-algorithm} terminates after a finite number of steps and returns a value $\tilde V$ such that $\vs - \tilde V \leq \varepsilon$, and a prefix $\tilde \Pi$ such that $\unv_{\tilde \Pi} = \tilde V$.
\end{theorem}

\begin{proof}[\proofof{thm:bb-algorithm-bounded-error}]
    We start by proving that \Cref{bb-algorithm} terminates after at most $\left\lceil \log_{p_{\max}} \frac{\varepsilon (1 - p_{\max})}{p_{\max}} \right\rceil$ steps. For every prefix $\Pi$ of length $H$ it holds that:
    \begin{align*}
        \ovv - \unv &= \sum_{t=1}^{H} \prod_{i=1}^{t} p_{\pi_i}(\bb^{\pi, \bq}_i) + \prod_{i=1}^{H} p_{\pi_i}(\bb^{\pi, \bq}_{i}) \cdot V^{U}(\bb^{\pi, \bq}_{H+1}) - \sum_{t=1}^{H} \prod_{i=1}^{t} p_{\pi_i}(\bb^{\pi, \bq}_i) - \prod_{i=1}^{H} p_{\pi_i}(\bb^{\pi, \bq}_{i}) \cdot V^{L}(\bb^{\pi, \bq}_{H+1}) \\
        &= \prod_{i=1}^{H} p_{\pi_i}(\bb^{\pi, \bq}_{i}) \cdot V^{U}(\bb^{\pi, \bq}_{H+1}) - \prod_{i=1}^{H} p_{\pi_i}(\bb^{\pi, \bq}_{i}) \cdot V^{L}(\bb^{\pi, \bq}_{H+1}) \\
        &\leq \prod_{i=1}^{H} p_{\pi_i}(\bb^{\pi, \bq}_{i}) \cdot V^{U}(\bb^{\pi, \bq}_{H+1}) \\
        &\leq p_{\max}^H \cdot \frac{p_{\max}}{1-p_{\max}} \leq \varepsilon.
    \end{align*}
    Furthermore, notice every time the algorithm enters Line~\ref{bnbalg:branching} with a prefix $\Pi$ it holds that $\unv_{\Pi} \leq \tilde V$, as $\tilde V$ was updated earlier. Therefore, starting from depth $\left\lceil \log_{p_{\max}} \frac{\varepsilon (1-p_{\max})}{p_{\max}} \right\rceil+1$, no prefix will satisfy the condition, and the algorithm will terminate. After establishing that the algorithm terminates, we now show that it returns an $\varepsilon$-optimal solution. Assume towards contradiction that \Cref{bb-algorithm} returns a value $\tilde V$ such that $\vs - \tilde V > \varepsilon$. Denote the last prefix of $\ps$ that was pruned as $\Pi'$, and the value of $\tilde V$ that bounded it as $V'$. It holds that $\ovv_{\Pi'} < V' + \varepsilon$, as $\Pi'$ was bounded. Because $V' \leq \tilde V$, it holds that $\vs(\bq) \leq \overline{V_{\Pi'}} < \tilde V + \varepsilon$, which contradicts the assumption.
\end{proof}

\begin{remark}
  The output prefix $\tilde{\Pi}$ is a finite sequence of categories, while a policy is defined as an infinite sequence. We can extend $\tilde{\Pi}$ to an approximately optimal policy by using $\tilde{\Pi}$ for the first $h = |\tilde{\Pi}|$ rounds, and then repetitively picking the best-fixed category with respect to $\bb^{\Pi, \bq}_{h+1}$.
\end{remark}

\section{Synthetic Experiments}\label{sec:experiments}

In this section, we conduct experiments with two goals in mind. First, we complement our convergence result from \Cref{thm:convergence} by demonstrating that, in practice, the belief walk converges rather quickly. Second, using simulated data, we compare the performance of \Cref{bb-algorithm} with that of a state-of-the-art benchmark. 

\paragraph{Simulation details}
We generate instances using a random sampling procedure. We generate $\bP$ by independently sampling latent vectors for each user type and category from a normal distribution, computing entries of $\bP$ as negated cosine distances (representing a user's affinities to a category), and normalizing these entries. Probabilities are clipped to $[0.01, 0.99]$ for numerical stability. We generate $\bq$ by independently sampling logits from a normal distribution $\mathcal{N}(0, 0.5)$. Then, we transform them into a categorical distribution through the softmax function. All the simulations in this section were conducted on a desktop PC equipped with an 11th Gen Intel(R) Core i5-11600KF @ 3.90GHz processor and 16 GB of RAM.
, and only the CPU was used. Running the experiments took several hours.

\begin{figure}
   \centering
   \includegraphics[width=.5\linewidth]{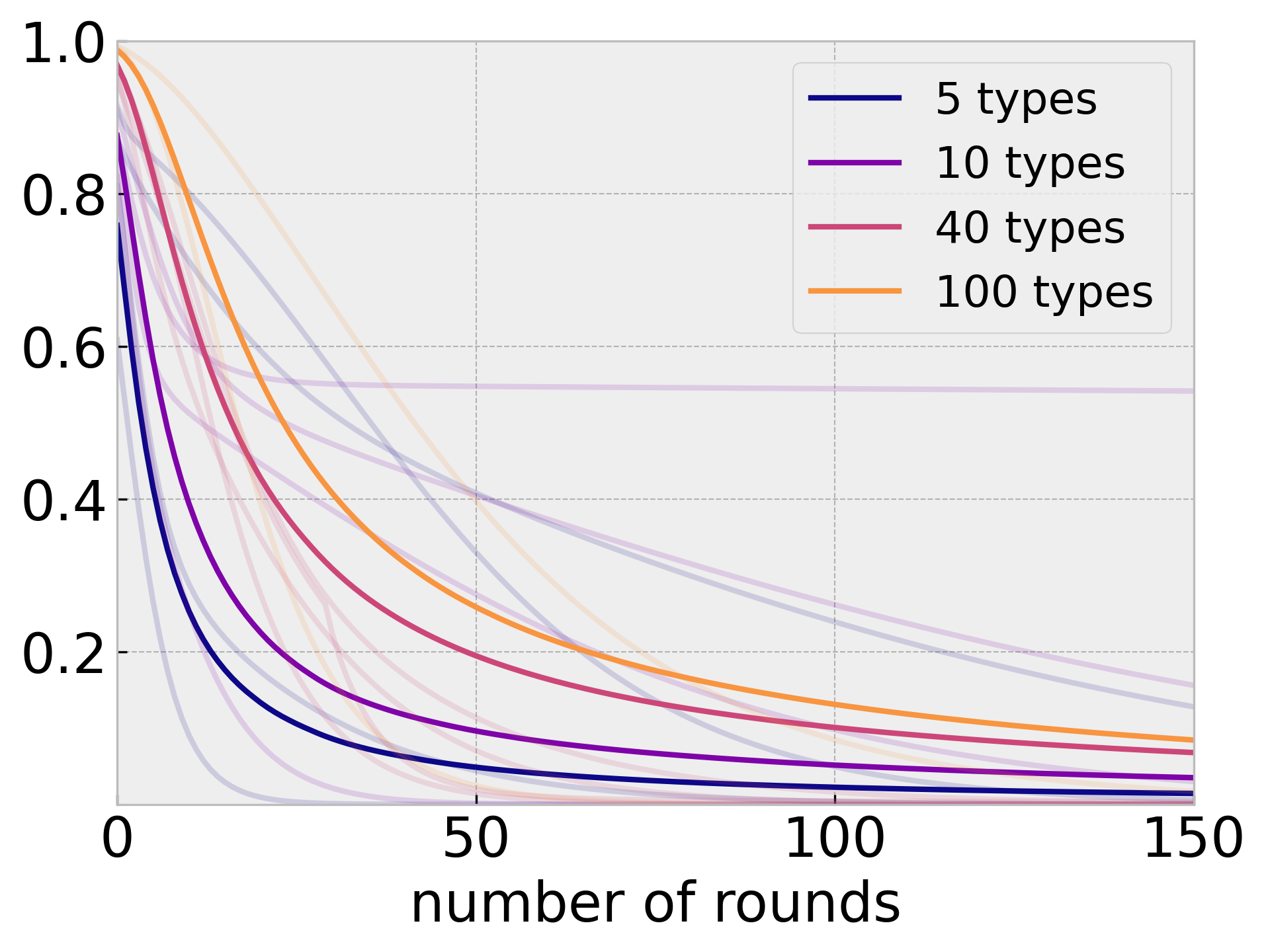}
   \caption{Convergence of beliefs under optimal policies. The number of categories is set to $10$. The x-axis is the number of rounds, and the y-axis is the uncertainty in the user type. The transparent lines illustrate individual runs, and the solid lines represent averages over $500$ runs.}
   \label{fig:uncertainty}
\end{figure}

\paragraph{Convergence of beliefs}
\Cref{thm:convergence} establishes that the optimal policy eventually converges to picking a fixed category. While the theorem guarantees convergence, it does not provide explicit rates. Equivalently, convergence can be analyzed in terms of certainty about a user's type, represented by proximity to the vertex to which the belief walk converges (recall the proof of \Cref{thm:convergence}). Since beliefs update according to Bayes' rule, they converge at a geometric rate once the policy becomes fixed. In other words, further exploration yields diminishing returns when a belief is sufficiently close to a vertex. Thus, it is tempting to assume that the optimal policy myopically maximizes value for that vertex. On the other hand, a poorly chosen myopic policy can fail drastically, as \Cref{prop:myopic-policy-suboptimality} illustrates. We resolve these conflicting observations through simulations.

\Cref{fig:uncertainty} shows how \emph{uncertainty} in user type, defined as the $l_1$-distance from the vertex to which the belief converges under the optimal policy, evolves throughout the session. We vary the number of user types while keeping the number of categories constant. For each problem size, we report the average uncertainty and several individual runs. Despite the heterogeneity of individual runs, their geometric convergence property roughly transfers to averaged curves: Exponential functions fitted to these curves are almost identical to the originals, with correlation coefficients of at least $R^2=0.98$. This aligns with our intuition that early rounds are most important in terms of both expected reward and information.

Analyzing individual runs reveals notable patterns. While in some sessions, the optimal policies are fixed from the start, in others, recommendations switch (as characterized by jumps in the slope). This reflects the short-term versus long-term reward trade-off discussed throughout the paper: The optimal policy may initially prioritize immediate rewards before switching to a riskier recommendation that increases certainty and yields greater returns in the long run. 
Despite this, all the presented curves strictly decrease, suggesting that certainty increases monotonically. However, we found that in rare cases, the optimal policy can move away from a vertex before converging to it.
This resolves the above conflict: Even if the belief approaches some vertex, the optimal policy may eventually lead to a different vertex. We exemplify this behavior in \Cref{sec:belief-walks}.

\begin{figure*}
\centering
\begin{subfigure}{0.02\textwidth}
   \centering
   \includegraphics[width=\linewidth]{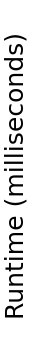}
   \vspace{0.7cm} % Add vertical space
\end{subfigure}%
\hspace{0.1cm} % Add horizontal space between subfigures
\begin{subfigure}{0.3\textwidth}
   \centering
   \includegraphics[width=\linewidth]{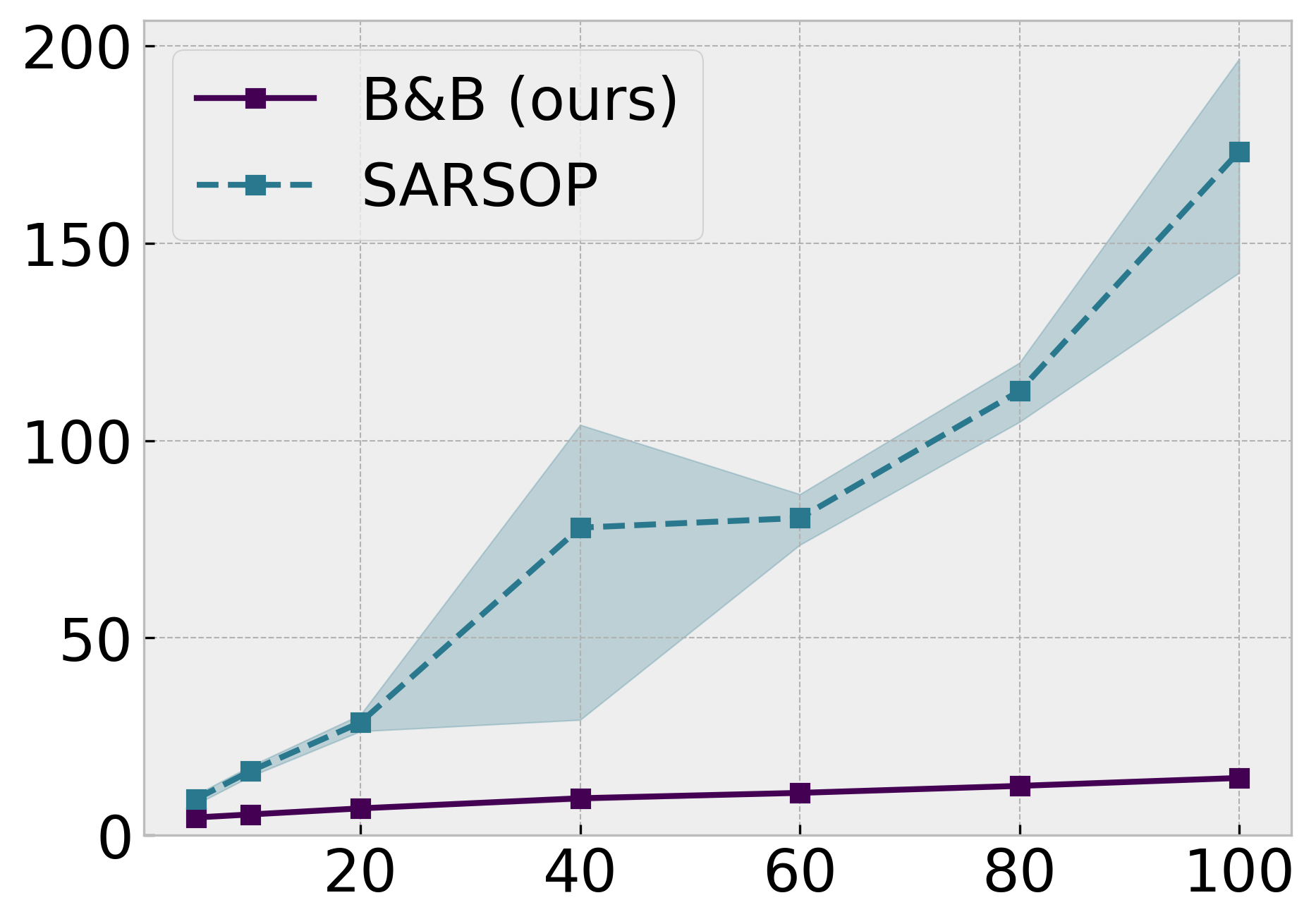}
   \caption{x-axis varies types; 10 categories}
   \label{fig:types_time}
   \vspace{0.5cm} % Add vertical space
\end{subfigure}%
\hspace{0.2cm} % Add horizontal space between subfigures
\begin{subfigure}{0.307\textwidth}
   \centering
   \includegraphics[width=\linewidth]{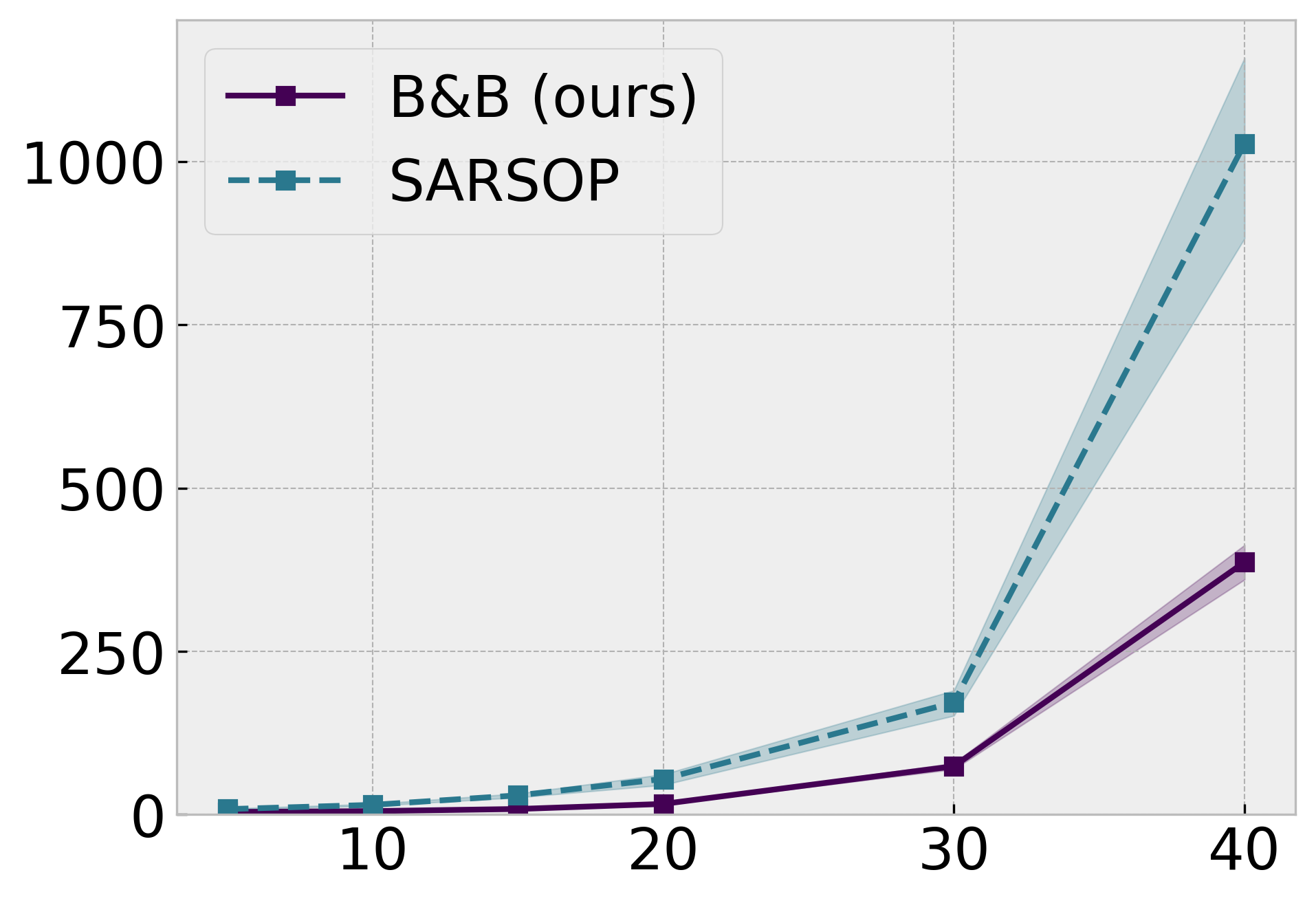}
   \caption{x-axis varies types and categories}
   \label{fig:actions_types_time}
   \vspace{0.5cm} % Add vertical space
\end{subfigure}%
\hspace{0.2cm} % Add horizontal space between subfigures
\begin{subfigure}{0.3\textwidth}
   \centering
   \includegraphics[width=\linewidth]{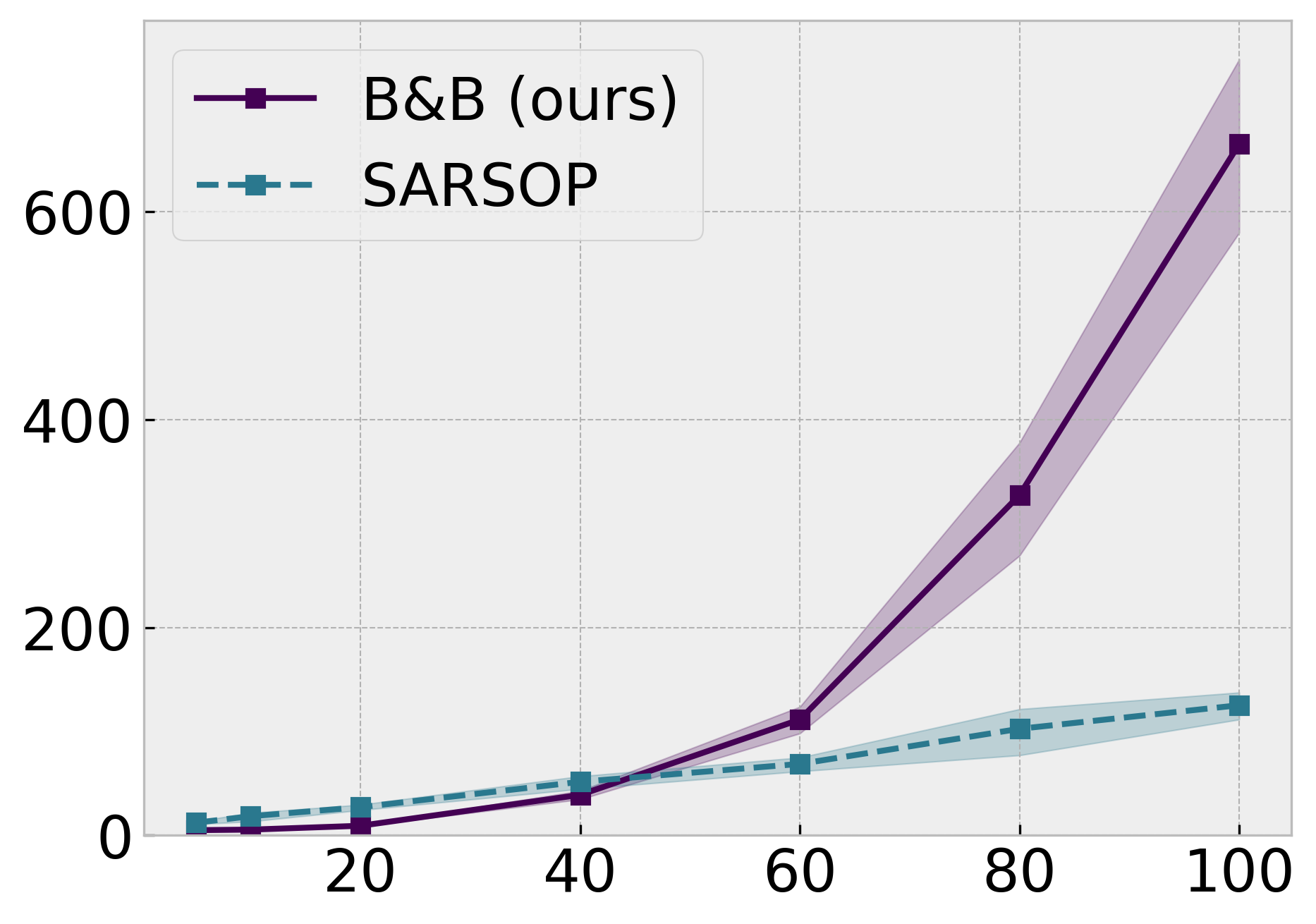}
   \caption{x-axis varies categories; 10 types}
   \label{fig:actions_time}
   \vspace{0.5cm} % Add vertical space
\end{subfigure}%
\caption{Runtime comparison (in milliseconds) between \Cref{bb-algorithm} and a POMDP solver SARSOP on randomly generated synthetic data. Each data point represents the average runtime over $500$ instances of $\prob$. Shaded intervals represent $95\%$ bootstrap confidence intervals of the empirical average. Both algorithms stop when they reach a precision of $\varepsilon=10^{-6}$.\label{fig:sarsop}}
\end{figure*}

\paragraph{Runtime comparison}
Our model is novel, so there are no specifically tailored baselines. However, since it can be cast as a POMDP, we can compare \Cref{bb-algorithm} with more general solvers. As a baseline, we have chosen SARSOP, a well-known offline point-based POMDP solver \citep{kurniawati2009sarsop}. While \Cref{bb-algorithm} is straightforward, 
SARSOP is rather complex. It represents the optimal policy through $\alpha$-vectors (a convex piece-wise linear approximation of the value function) and clusters sampled beliefs to estimate the values of new ones. We used an open-source implementation of SARSOP,\footnote{\url{https://github.com/AdaCompNUS/sarsop}}, and, for a fair comparison, implemented \Cref{bb-algorithm} in the same language. Importantly, both algorithms are exact: By construction, they compute the same optimal policy for a given instance up to the specified precision. Therefore, the comparison we present next reflects differences solely in computational efficiency rather than solution quality. 

\Cref{fig:sarsop} presents the runtime comparison between \Cref{bb-algorithm} and SARSOP. While \Cref{bb-algorithm} dominates SARSOP on rectangular problems with a few categories (\Cref{fig:types_time}) and overperforms SARSOP on square problems (\Cref{fig:actions_types_time}), it underperforms when the number of categories is much higher than the number of user types (\Cref{fig:actions_time}). Statistical significance of these performance differences is confirmed using the Wilcoxon signed-rank test, with all comparisons yielding $p$-values less than $10^{-10}$.

We hypothesize that this result is due to SARSOP more effectively handling similar categories (similar associated rows in the matrix $\bP$) through the $\alpha$-vector representation and clustering heuristic. Another explanation is that \Cref{bb-algorithm} explores the policy space; hence, the branching factor is the number of categories. In contrast, SARSOP explores the belief space, whose dimension is the number of user types. Consequently, we could expect SARSOP to struggle in cases with many user types and \Cref{bb-algorithm} to encounter challenges in cases with many categories. Overall, each algorithm excels under different conditions.

\section{Experiments with MovieLens Data}\label{sec:movieLens}

In this section, we demonstrate the applicability of our approach to real-world data by conducting experiments on the MovieLens 1M dataset~\citep{harper2015movielens}. We start by describing the necessary steps to obtain an instance of our model from standard recommender system data, and then evaluate the performance of \Cref{bb-algorithm} against SARSOP on this obtained instance. The experiments were conducted on a MacBook equipped with an Apple M3 processor and 18 GB of RAM, using only the CPU. Running the simulations took approximately one hour.

\subsection{Constructing an $\prob$ Instance}

The MovieLens 1M dataset contains 1 million ratings from 6040 users on 3706 movies, where each rating is an integer value between 1 and 5. This is a standard representation of a recommender system problem, namely a sparse user-item rating matrix, where observed entries represent user ratings for items. Recall that our model consists of a dense probability matrix between user types and content categories, accompanied by a prior distribution over user types. Hence, our model cannot be applied directly, and the following two transformations are required.

\paragraph{Deriving categories and user types.} First, the sparse rating matrix must be aggregated into a concise representation through clustering of users and items. The task of clustering in recommender systems is well-established, with numerous methods proposed in the literature (e.g.,\citep{coclustering, recommender-systems-for-large-scale-e-commerce-scalable-neighborhood-formation-using-clustering, gong2010collaborative, xu2012exploration, a-clustering-approach-for-personalizing-diversity-in-collaborative-recommender-systems}). In our implementation, we adopt spectral co-clustering, which jointly clusters users and items based on the structure of their interactions~\citep{coclustering}. We note this method typically returns the same number of clusters for users and items, which may be inappropriate for certain applications. In such cases, a secondary clustering procedure may be applied to refine the granularity of either partition.

\paragraph{Estimating preferences and prior distribution.} After deriving the categories and user types through clustering, the clustered data must be used to obtain the prior distribution and the probability matrix. We naturally construct the prior $\bq$ using the cluster size as a proxy. Formally, the prior probability of each user type is the number of users assigned to that type divided by the total number of users in the dataset. For the probability matrix $\bP$, we compute the mean ratings within cluster pairs and normalize them to $[0,1]$. While this is a straightforward approach, alternative transformations can be employed, such as using percentile rankings or more sophisticated normalization techniques. 

\subsection{Simulations Setup and Results}

After outlining the pipeline for transforming conventional recommender systems into $\prob$ instances, we compare \Cref{bb-algorithm} with SARSOP across instances extracted from the MovieLens dataset. We vary the instance size by adjusting the number of clusters forwarded to the Spectral co-clustering algorithm. To assess the robustness of the comparison, we tested both algorithms on 500 randomly generated noise samples for each matrix size. Specifically, we added Gaussian noise ($\sigma=0.005$) to the extracted probability matrices and evaluated the algorithms' performance on these perturbed instances.

The results of the comparison are presented in \Cref{fig:movielens}. The resulting 95\% bootstrap confidence intervals are presented as shaded regions around the mean curves, though they appear narrow due to the consistent runtime performance across experimental runs. As can be seen, \Cref{bb-algorithm} consistently outperforms SARSOP across all matrix dimensions, with the performance gap increasing with the problem's size. Statistical significance is confirmed using the Wilcoxon signed-rank test, which yields $p$-values less than $10^{-10}$ for all comparisons. This aligns with the results of the synthetic experiments presented in \Cref{sec:experiments} (Figure~\ref{fig:actions_types_time}), demonstrating that our algorithm remains effective even when applied to noisy, real-world data.

\begin{figure*}
\centering
\begin{subfigure}{0.03\textwidth}
   \centering
   \includegraphics[width=\linewidth]{simulations/time_comp_y_axis_caption.png}
   \vspace{0.7cm} % Add vertical space
\end{subfigure}%
\hspace{0.1cm} % Add horizontal space between subfigures
\begin{subfigure}{0.45\textwidth}
   \centering
   \includegraphics[width=\linewidth]{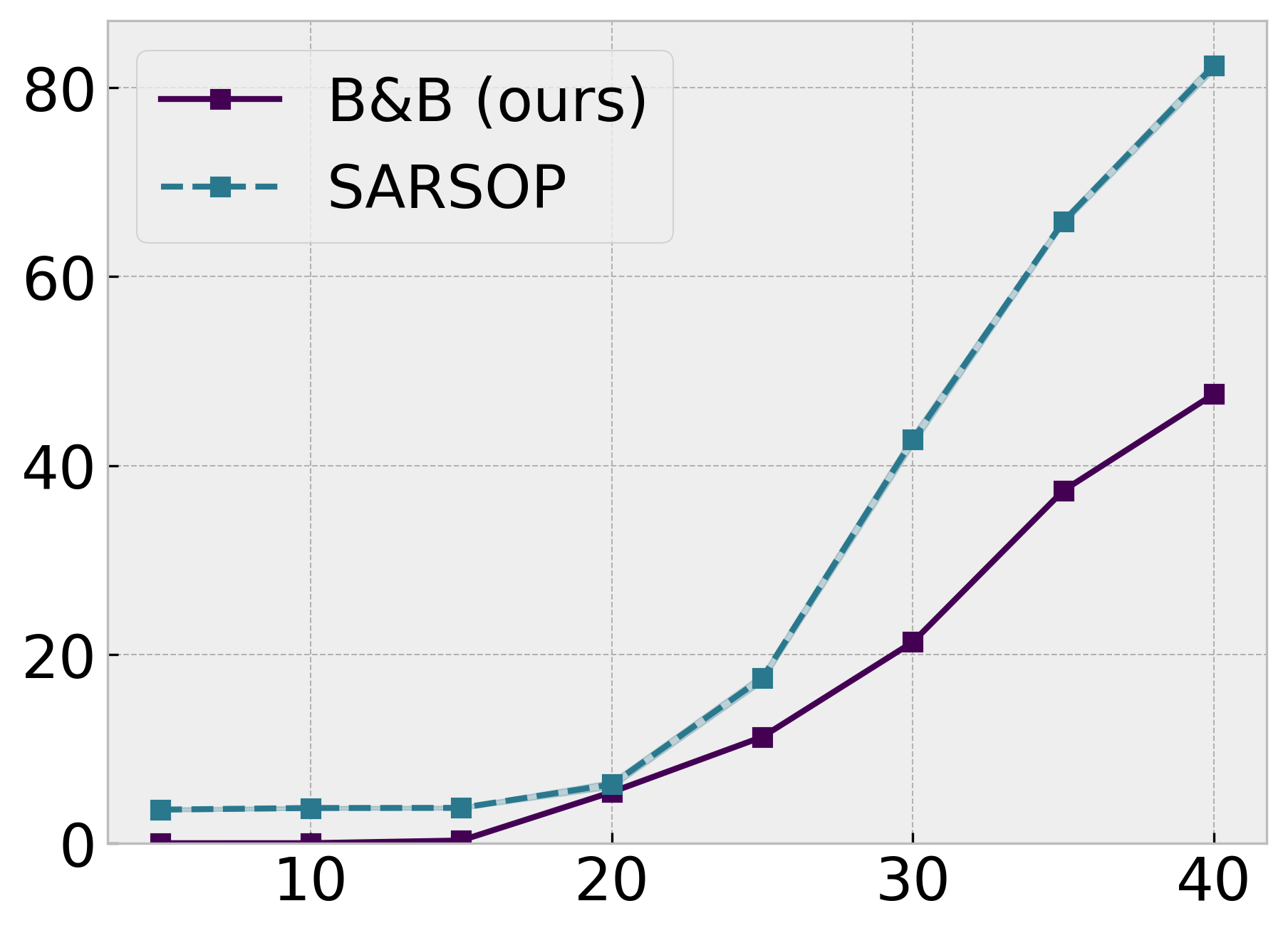}
   \caption{x-axis varies types and categories}
   \vspace{0.5cm} % Add vertical space
\end{subfigure}%
\caption{Runtime comparison (in milliseconds) between \Cref{bb-algorithm} and a POMDP solver SARSOP on MovieLens-derived data with added noise. Each data point represents an average runtime over 500 $\prob$ instances with Gaussian noise ($\sigma= 0.005$) added to probability matrices extracted from the MovieLens dataset. Shaded intervals represent 95\% bootstrap confidence intervals of the empirical average. Both algorithms terminate when they achieve a precision of $\varepsilon = 10^ {- 6} $.\label{fig:movielens}}
\end{figure*}

\section{Discussion and Future Work}\label{sec:disc}

We have introduced a model that captures two intertwined decision-making challenges: Partial information with aggregated user data and the risk of churn. We proposed algorithms for rectangular instances as a warm-up, and then showed that optimal policies converge and eventually act greedily. We have proposed lower and upper bounds on the optimal social welfare and formulated a B\&B algorithm that uses them. Finally, we have demonstrated that our B\&B algorithm performs comparably to a state-of-the-art baseline on both synthetic and real-world instances.

We see several directions for future work.  First, despite the nontrivial analysis, we still lack either a provably optimal polynomial-time algorithm for computing an optimal policy or a formal proof of hardness. A resolution in either direction would be informative: A polynomial-time algorithm would guide the design of efficient methods, while hardness results would clarify what makes certain instances difficult to solve.

Second, while our model provides a clear framework for analyzing the structure of optimal decision-making under uncertainty and user churn, it assumes a dichotomous structure: Either the user likes the item and remains in the system or dislikes it and leaves. In practice, users may stay in the system after receiving negative feedback or provide feedback with varying intensity. These richer dynamics can be modeled within a more general POMDP framework by enriching the transition, observation, and reward functions, while preserving the core components of our model--namely, the abstraction to content categories and user types, and uncertainty over user type. In such a case, we can still apply Bayesian updates after receiving feedback from the user; however, a new approach should be adopted, as belief walks become more complex since different feedback leads to a different updated belief. This technical difficulty undermines both our theoretical results and our algorithm design, which rely on deterministic updates. Future work in this direction could involve designing structure-aware POMDP solvers for practical use or developing more sophisticated analytical tools to establish (or rule out) convergence in probabilistic belief dynamics.

\bibliographystyle{plainnat}
\bibliography{main}

\appendix

\section{Showcasing Belief Walks under Various Policies} \label{sec:belief-walks}
In this section, we illustrate the structure of belief walks under various policies, highlighting the distinct behaviors and trade-offs that emerge in the context of belief walks. Each graph corresponds to a different instance with three categories and three user types ($\abs{M} = \abs{K} = 3$). We provide the full details of the instances below. In each graph, we plot the belief walk induced by three policies. Each belief in the 3-dimensional simplex is characterized by a probability for the first type $\bm b(m_1)$, given in the x-axis, probability in the second type $\bm b(m_2)$, given in the y-axis, and probability for the third type $\bm b(m_3)$, given implicitly by $\bm b(m_3) = 1 -\bm b(m_1) - \bm b(m_2)$. The plotted policies are:
\begin{enumerate}
\item The optimal policy, calculated using Algorithm~\ref{bb-algorithm}.
\item The best fixed-action policy, which always recommends the category that maximizes the value function over all single-action policies with respect to the prior. Put formally - \[ \pi^f = \left( \tilde k \right)_{t=1}^{\infty}, \quad \tilde k = \argmax_{k \in K} \sum_{m \in M} \bm{q}(m) \cdot \frac{\bm{P}(k, m)}{1 - \bm{P}(k, m)}. \]
For the remainder of the section, we abbreviate "best fixed-action policy" as "BFA policy."
\item The myopic policy, which always recommends the category that yields the highest immediate reward in the current belief and then updates its belief afterward. Formally, \[ \pi^m = \left( k_t \right)_{t=1}^{\infty}: k_i = \argmax_{k \in K} \sum_{m \in M} \bm{b_i}(m) \cdot \bm{P}(k, m), \ \ \bm b_1 = \bm q, \ \ \forall i > 1: \bm b_i = \tau(\bm b_{i-1}, k_{i-1}). \]
\end{enumerate}

Each belief in the belief path produced by the optimal policy is assigned a numerical value based on its sequence position. The plots display these values as small black numbers beside some arrows. Each arrow signifies a Bayesian update followed by a recommendation, with the arrow's direction indicating the movement from the previous belief to the updated belief. The black dot depicts the initial prior $\bm q$ over the user type, marking the starting point for any plotted belief walk.

\begin{figure}
\centering
\begin{subfigure}{0.45\textwidth}
    \includegraphics[width=\linewidth]{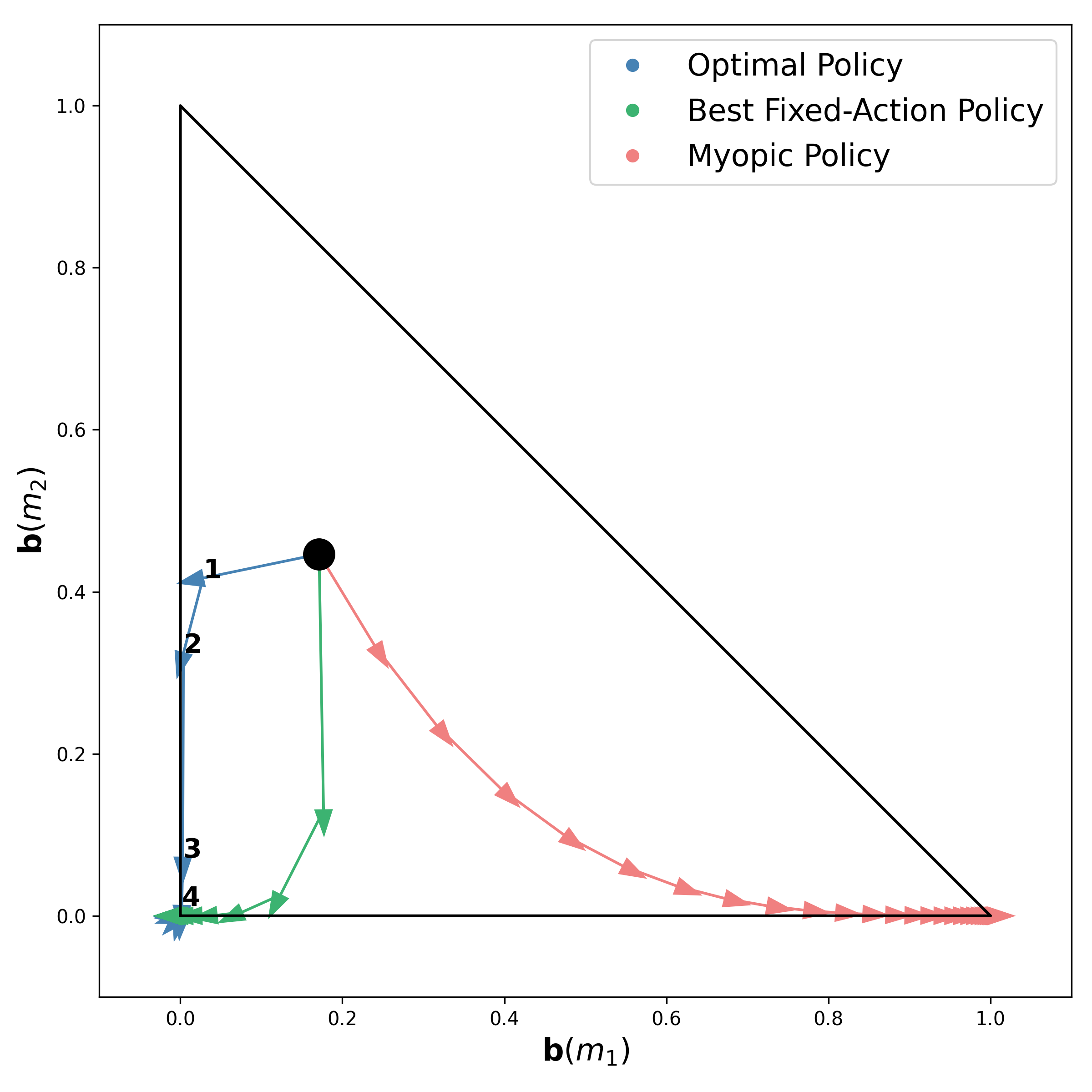}
    \caption{Figure accompanying Example~\ref{example:1} - the optimal policy avoids myopic behavior.}
    \label{fig:belief-walk1}
\end{subfigure}
\hfill
\begin{subfigure}{0.45\textwidth}
    \includegraphics[width=\linewidth]{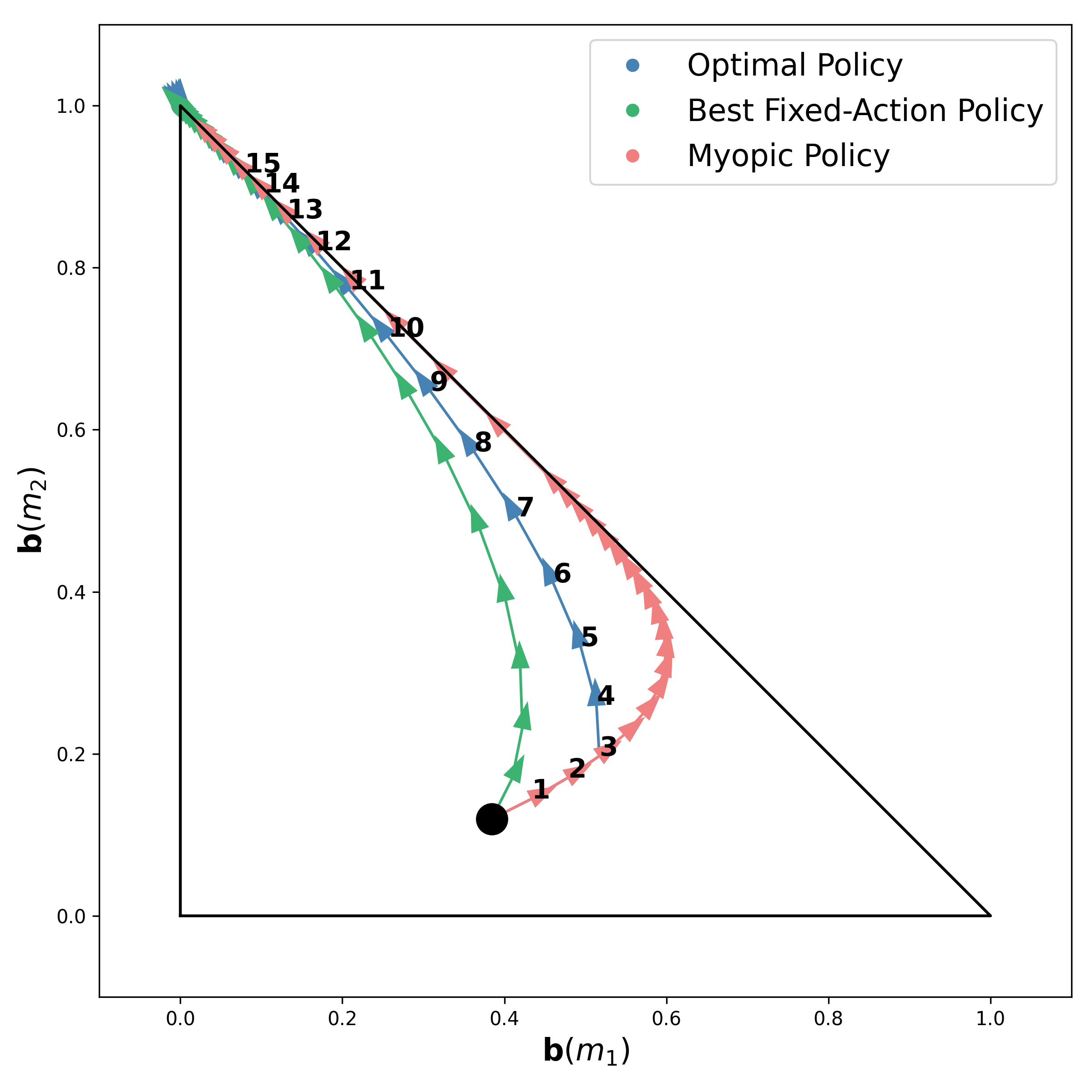}
    \caption{Figure accompanying Example~\ref{example:2} - all policies converge to the same vertex.}
    \label{fig:belief-walk2}
\end{subfigure}

\vspace{1em}

\begin{subfigure}{0.45\textwidth}
    \includegraphics[width=\linewidth]{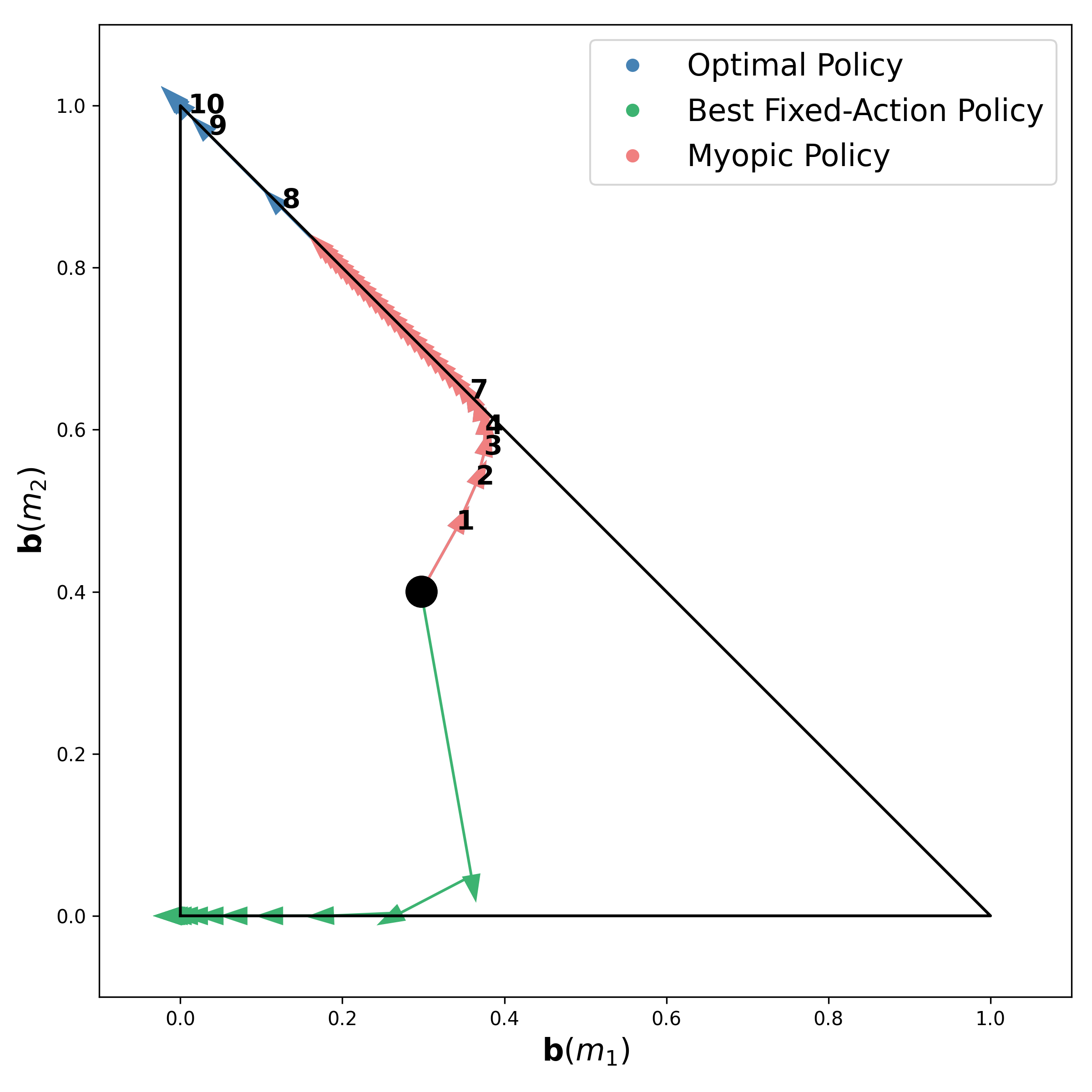}
    \caption{Figure accompanying Example~\ref{example:3} - the optimal policy diverges from the BFA policy prefix and forsakes the myopic policy for faster convergence.}
    \label{fig:belief-walk3}
\end{subfigure}
\hfill
\begin{subfigure}{0.45\textwidth}
    \includegraphics[width=\linewidth]{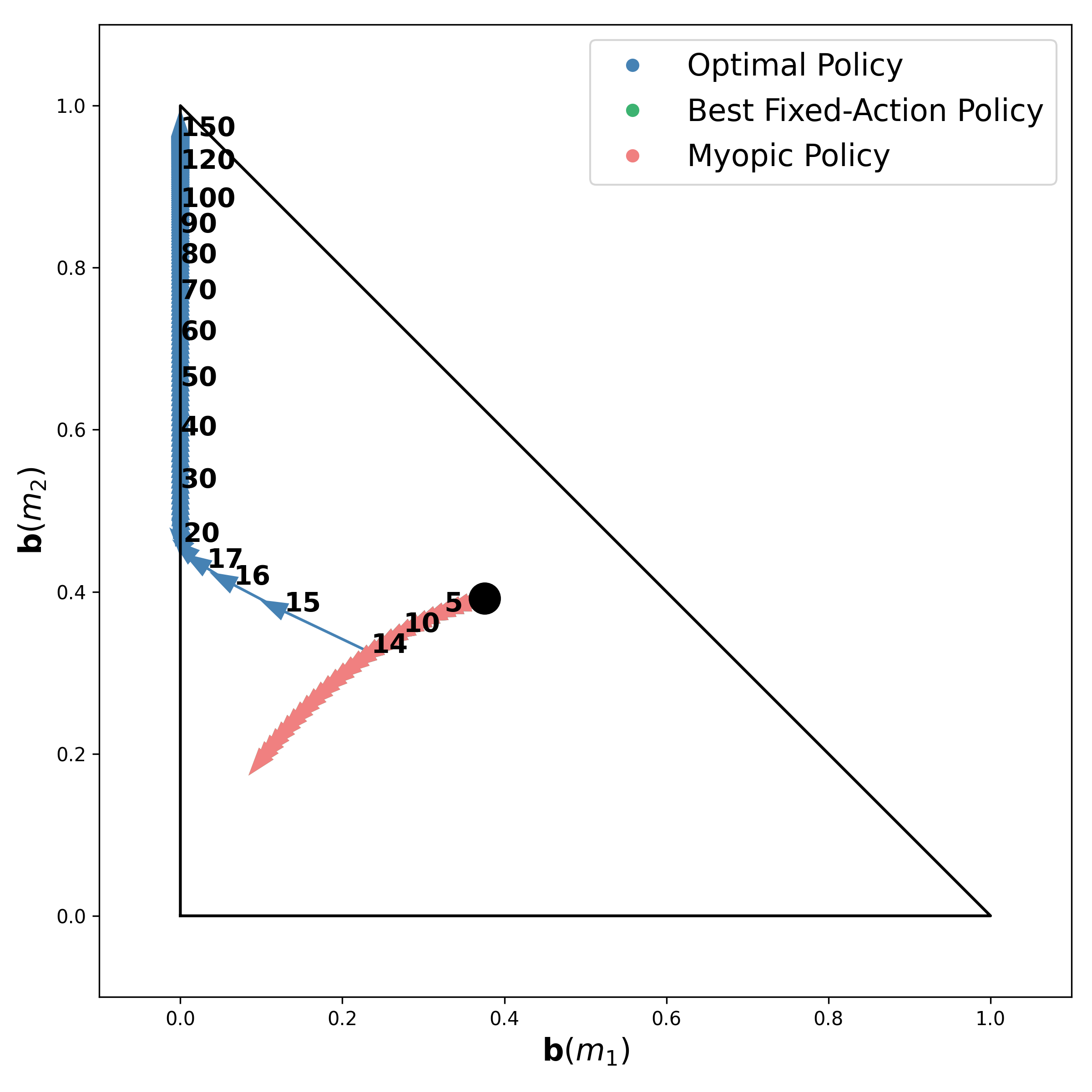}
    \caption{Figure accompanying Example~\ref{example:4} - there is a decrease followed by an increase in the belief of the type the belief walk converges to.}
    \label{fig:belief-walk4}
\end{subfigure}
\caption{Belief walks under different policies}
\label{fig:belief-walks}
\end{figure}

\begin{example} \label{example:1}
The $\prob$ instance in Figure~\ref{fig:belief-walk1} is $\langle K = \{ k_1, k_2, k_3 \},\ M = \{ m_1, m_2, m_3 \},\ \bm{q}_1,\ \bm{P}_1 \rangle$, where:
\[
    \bm{q}_1 = \begin{pmatrix} 0.1713 \\ 0.4465 \\ 0.3822 \end{pmatrix},\ 
    \bm{P}_1 = \begin{pmatrix} 0.8611 & 0.4591 & 0.6862 \\ 0.0969 & 0.5531 & 0.8604 \\ 0.5055 & 0.1430 & 0.8879 \\ \end{pmatrix}.
\]
In this instance, the optimal and BFA policies converge to the user type $m_3$, whereas the myopic policy converges to $m_1$. This observation underlines the central idea of Proposition~\ref{prop:myopic-policy-suboptimality}, which asserts that the myopic policy may overlook future rewards, sometimes leading to convergence to a suboptimal user type. Even though the optimal and BFA policies reach the same user type, their pathways differ. This underscores the importance of meticulous initial exploration, as the sequence of recommendations before convergence significantly influences the overall rewards.
\end{example}

\begin{example} \label{example:2}
The $\prob$ instance in Figure~\ref{fig:belief-walk2} is $\langle K = \{ k_1, k_2, k_3 \},\ M = \{ m_1, m_2, m_3 \},\ \bm{q}_2,\ \bm{P}_2 \rangle$, where:
\[
    \bm{q}_2 = \begin{pmatrix} 0.3844 \\ 0.1197 \\ 0.4959 \end{pmatrix},\
    \bm{P}_2 = \begin{pmatrix} 0.6848 & 0.9100 & 0.5457 \\ 0.7741 & 0.8284 & 0.5833 \\ 0.1931 & 0.9127 & 0.5273 \\ \end{pmatrix}.
\]
In this instance, all three policies converge to $m_2$, albeit through different paths. Initially, for the first three recommendations, the optimal policy aligns with the myopic policy. Later, it diverges to secure higher future rewards, considering that it might get lower immediate rewards.
\end{example}

\begin{example} \label{example:3}
The $\prob$ instance in Figure~\ref{fig:belief-walk3} is $\langle K = \{ k_1, k_2, k_3 \},\ M = \{ m_1, m_2, m_3 \},\ \bm{q}_3,\ \bm{P}_3 \rangle$, where:
\[
    \bm{q}_3 = \begin{pmatrix} 0.2972 \\ 0.4001 \\ 0.3027 \end{pmatrix},\
    \bm{P}_3 = \begin{pmatrix} 0.5492 & 0.0560 & 0.8878 \\ 0.2195 & 0.8576 & 0.2072 \\ 0.7674 & 0.7992 & 0.4051 \\ \end{pmatrix}.
\]
In this instance, the BFA policy deviates from optimal and myopic policies early. Initially, these two policies align in their recommendations until the optimal policy transitions from belief (7) to belief (8), as shown in the graph. Although this recommendation provides a lower immediate reward than the myopic recommendation, it leads to a belief state much closer to $m_2$ than the updated belief after the myopic recommendation. This emphasizes the long-term rewards in this scenario, highlighting the common trade-off between short-term and long-term goals in sequential decision-making.
\end{example}

\begin{example} \label{example:4}
The $\prob$ instance in Figure~\ref{fig:belief-walk4} is $\langle K = \{ k_1, k_2, k_3 \},\ M = \{ m_1, m_2, m_3 \},\ \bm{q}_4,\ \bm{P}_4 \rangle$, where:
\[
    \bm{q}_4 = \begin{pmatrix} 0.3755 \\ 0.3921 \\ 0.2324 \end{pmatrix},\
    \bm{P}_4 = \begin{pmatrix} 0.4011 & 0.8521 & 0.8301 \\ 0.7683 & 0.7837 & 0.8314 \\ 0.7674 & 0.7832 & 0.4051 \\ \end{pmatrix}.
\]
In this instance, the BFA and myopic policies coincide and act optimally until the 14th step, at which point they diverge from the optimal policy. This figure is an example of the phenomenon described in Section~\ref{sec:experiments}. In the first 14 steps, the entry in the belief walk that corresponds to user type $m_2$ decreases, and one could expect the belief walk to approach a different vertex. Then, in step 15, as the optimal policy diverges from the other two policies, we see a sudden increase in $\bm{b}(m_2)$, which persists throughout the remainder of the belief walk. In conclusion, it is possible to see a decrease before the increase in the certainty of the user type to which the belief walk finally converges.
\end{example}

\end{document}